\documentclass[12pt]{amsart}
\usepackage[margin=3cm]{geometry}
\RequirePackage{amsmath, amssymb, amsthm, amsfonts, amscd}
\RequirePackage{indentfirst}
\RequirePackage[pdftex]{graphicx}
\RequirePackage{hyperref}
\RequirePackage[dvipsnames]{xcolor}
\usepackage{algorithm}
\usepackage{algorithmic}
\usepackage{citeref}
\usepackage{mathrsfs}

\numberwithin{equation}{section}

\newtheorem{theorem}{Theorem}[section]
\newtheorem{lemma}[theorem] {Lemma} 
\newtheorem{corollary}[theorem]{Corollary} 
\numberwithin{claim}{section}

\newtheorem{assumption}[theorem]{Assumption}
\newtheorem{remark}[theorem]{Remark}  

\theoremstyle{definition}
\numberwithin{definition}{section}
\newtheorem*{definition*}{Definition}

\numberwithin{example}{section}

\newcommand{\norm}[1]{\left\lVert#1\right\rVert}
\newcommand{\rb}[1]{\left(#1\right)}
\newcommand{\cb}[1]{\left\{#1\right\}}

\newcommand{\R}{\mathbb{R}}
\newcommand{\E}{\mathbb{E}}

\newcommand{\B}{\mathcal{B}}

\title{Barron Space for Graph Convolution Neural Networks}

\author{Seok-Young Chung and Qiyu Sun}
\address{Chung: Department of Mathematics, University of Central Florida, Orlando, FL 32765, USA}
\email{Seok-Young.Chung@ucf.edu}
\address{Sun: Department of Mathematics, University of Central Florida, Orlando, FL 32765, USA}
\email{qiyu.sun@ucf.edu}

\date{}

\begin{document}

\begin{abstract}
Graph convolutional neural network  (GCNN) operates on graph domain and it  has  achieved a superior performance to accomplish a wide range of tasks.
In this paper, we introduce a Barron space  of functions on a compact domain of graph signals. We prove that the proposed Barron space
is a reproducing kernel Banach space, it can be decomposed into the union of
 a family of reproducing kernel Hilbert spaces with neuron kernels, and it could be dense in the space of continuous functions on the domain.
Approximation property is one of the main principles
to design neural networks. In this paper, we  show that outputs of GCNNs are contained in the Barron space  and functions in the Barron space can be well approximated
 by outputs of some GCNNs in the integrated square and uniform measurements. 
 We also estimate the Rademacher complexity of functions with bounded Barron norm and conclude that functions in the Barron space could be
 learnt from their random samples efficiently.
\end{abstract}

\maketitle

\section{Introduction}

  Graph signal processing provides an innovative framework
to extract knowledge from  massive data sets residing on irregular domains and networks
\cite{Cheung2020, Dong2020, ncjs22, Isufi2023, Ortegabook, Ortega18, Ricaud2019,  aliaksei13, sandryhaila14, Shuman2013, Stankovicchapter2019}.
 Graphs  are widely used to model the  topological structure of irregular domains and networks. For instance,
 a sensor network can be described by a graph with vertices representing
sensors of the network and edges between vertices showing  peer-to-peer communication link  between sensors or being within certain range of a spatial domain  (usually indicating correlation among data collected), and  the skeleton structure of human body is naturally structured as a graph  with  joints as
vertices and their natural connections in the human body
as edges
\cite{Akoglu2015, chong2003, Shi2019,  Wass94, Yick08}.
 In this paper, we  consider
 weighted undirected connected graphs  $\mathcal{G}=({V}, E)$  of finite order $N\ge 1$.

Convolutional neural network (CNN) is one of the most representative  neural networks for machine learning.
It has gained a lot of  attention  from  industrial and academic communities, and it has made numerous exciting achievements. For instance,
  computer vision based on CNNs  makes it possible to accomplish tasks, such as face recognition, autonomous vehicle and intelligent medical treatment. The reader may refer to    \cite{Gu2018, Krizhevsky2017, LeCun1989,  Li2021, Niepert2016} and  references therein for historical remarks and recent advances.

Graph convolutional neural network  (GCNN) is the
generalization of classical CNNs that operates on graph domain and it has  achieved a superior performance to accomplish a wide range of tasks
\cite{Bronstein2017, Bruna2014, Defferrard2016, Dong2020,  Isufi2023, Kipf2017, Li2018, Liubook2020, Shi2019, Shuman2013, Wu2021, Yang2021, Zhang2019,  Zhou2020, Zhuang2018}.
However, a feasible extension of CNNs from regular grid (such as pixel grids to represent images)
to irregular graph (such as the skeleton structure of human body) is not straightforward.

GCNN takes advantage of topological structure of the underlying graph and  aggregates node information from the neighborhoods in a convolutional fashion.
A basic question is how to define graph convolution appropriately.
Two conventional approaches have been proposed to define graph convolution, one
from the spectral perspective while the other from the spatial perspective.
In the first trial to define a GCNN, the spatial convolution is used to sum up the neighboring features \cite{Bruna2014}.
In this paper, we adopt the spectral approach and use graph Fourier transform to define graph convolution  ${\bf b}*{\bf x}$ between two graph signals ${\bf b}$ and ${\bf x}$, see \eqref{convolution.def}.

Let  ${\mathcal W}$ be the linear space of graph signals  ${\bf b}$ used for the convolution in
GCNNs with its dimension denoted by $\dim {\mathcal W}\le N$, and
 $\sigma(t)=\max(0, t)$ be the ReLU activation function
with $\sigma({\bf y}), {\bf y}\in {\mathbb R}^N$, being defined componentwisely.
In this paper, we consider  shallow (two-layer) GCNNs
 with  graph signal input ${\bf x}$ in a  compact domain $\Omega$ and scalar-valued
   output  given by
\begin{equation} \label{fM.def}
f_M({\bf x}, {\pmb \Theta} )=\frac{1}{M} \sum_{m=1}^M {\bf a}_m^T \sigma  ({\bf b}_m*{\bf x}+{\bf c}_m),\  {\bf x}\in \Omega,\end{equation}
where ${\pmb \Theta}=({\pmb \theta}_1, \ldots, {\pmb \theta}_M)$
and
${\pmb \theta}_m=({\bf a}_m, {\bf b}_m, {\bf c}_m)\in  {\mathbb R}^N\times {\mathcal W}\times {\mathbb R}^N, 1\le m\le M$.
The above shallow GCNN has  neurons
$ \phi({\bf x}, {\pmb \theta}_m)= {\bf a}_m^T \sigma  ({\bf b}_m*{\bf x}+{\bf c}_m), 1\le m\le M$,
and $(2N+  \dim {\mathcal W})M$ parameters to learn.
Let ${\bf S}_1, \ldots, {\bf S}_K$ be commutative graph shifts, see Section \ref{graphshifts.section}.
For the case that
the convolution space ${\mathcal W}$ contains graph signals with the corresponding convolution operation can be
implemented by some polynomial filtering procedure with polynomials  up to a given degree $L$, i.e.,
${\bf b}_m*{\bf x}=h_m({\bf S}_1, \ldots, {\bf S}_K) {\bf x}$
for some multivariate polynomials $h_m, 1\le m\le M$, of degrees at most $L$,
the above shallow GCNN  is essentially the   ChebNet in the literature \cite{Defferrard2016, Isufi2023, Kipf2017,   Wu2021}.

\medskip

 In this paper, we consider the problem  whether and how a function  $f$ on the domain  $\Omega$ of (sparse) graph signals can be approximated by outputs of  shallow GCNNs, i.e.,
$$ f({\bf x})\approx  f_M({\bf x}, {\pmb \Theta}),\ {\bf x}\in \Omega$$
for some ${\pmb \Theta}\in ({\mathbb R}^N\times {\mathcal W}\times {\mathbb R}^N)^M$.

\medskip

We say that $A=O(B)$ for two  quantities $A$ and $B$ if $|A/B|$ is bounded by some absolute constant. The main contributions of this paper are as follows.

\begin{itemize}

\item[{(i)}] Barron space of functions on the unit cube $[0, 1]^N$
was introduced in \cite{Barron1993} with the help of Fourier transform.
 In this paper, we follow the spatial framework in \cite{Bach2017, EMaWu2022}  and  introduce a graph Barron space  ${\mathcal B}$ on a compact domain $\Omega$.  We show that the Barron space ${\mathcal B}$ is a reproducing kernel  Banach space and has Lipschitz property, see Theorem \ref{BanachBarronspace.thm} and Corollary \ref{lipschitz.cor}.

\item[{(ii)}] Reproducing kernel Hilbert spaces  (RKHSs) are ideal for function estimation,  and 
 their kernels are selected to measure certain similarity between  input data
\cite{Ghojogh2021, Rahimi2007, Scholkopf2002, Steinwart2008}.  In this paper, we decompose the graph Barron space  ${\mathcal B}$ into the union of a family of
RKHSs with  neuron kernels and establish their norm equivalence,
see Theorems \ref{rkhs.thm} and   \ref{BarronRkhs.thm}. 

\item[{(iii)}] The approximation of functions in  Barron/Besov/H\"older spaces by
outputs of some neural networks is well studied, see 
\cite{EMaWu2022, Shen2022, Siegel2023, Yang2023a, Yang2023b} and references therein.
 In this paper, we show that functions in the Barron space can be well approximated by some shallow GCNNs with bounded path norm, 
 and conversely the limit of outputs of shallow GCNNs with bounded path norm belongs to the Barron space, see
Theorems \ref{approximation.thm}, \ref{approximation.thm2} and \ref{inverseapproximaton.thm}. 
To reach integrated square (resp. uniform) approximation error of a function  $f$ in the Barron space  ${\mathcal B}$ by outputs of some shallow GCNNs at accuracy $\epsilon \|f\|_{\mathcal B}$,
we obtain from Theorems \ref{approximation.thm} and \ref{approximation.thm2} that the number of parameters in the shallow GCCN is about
$O(N\epsilon^{-2})$  (resp. $O(N^2 \epsilon^{-2} \log (1/\epsilon))$), where $N$ is the order of the underlying graph  ${\mathcal G}$ of the shallow GCNN and $\|f\|_{\mathcal B}$ is the Barron norm of the function
$f\in {\mathcal B}$.   As  expected,  the approximation of functions in the Barron space does not suffer from the curse of dimensionality (i.e., the order $N$ of the underlying graph
${\mathcal G}$ in the current setting).

\item[{(iv)}] Universal approximation theorem is one of the main principles to design  neural networks
\cite{Bruel2020, kerven2021, kerven2019, Pinkus1999}. In Theorem \ref{density.thm} and Corollary \ref{univeral.cor}, we establish the universal approximation theorem for  shallow GCNNs and density of the graph Barron space  ${\mathcal B}$, under some technical conditions on
the graph Fourier transform and the convolution space ${\mathcal W}$.

\item[{(v)}] Rademacher complexity of a function class is a conventional measure of generalization error in learning theory \cite{Bach2017, Bartlett2002, EMaWu2022,  Shai2014, Yin2019}. In this paper, we provide an estimate to the  Rademacher complexity of functions with bounded Barron norm, which depends on
    the negative square root of the sample size and the square root of the  logarithmic of the graph order, see Theorem \ref{Rademacher.thm}.
  As a consequence, we see that functions in the Barron space could be learnt from their random samples in an efficient way, see  Theorem  \ref{error.thm}.

\end{itemize}

The rest of this paper is organized as follows. In Section \ref{preliminaries.section},
 we recall some preliminaries on commutative graph shifts, graph Fourier transform, graph convolutions and  shallow GCNNs on a compact domain
 $\Omega$ of (sparse) graph signals.
In Section \ref{barronspace.section},  we introduce the graph Barron space ${\mathcal B}$
and a family of  graph reproducing kernel Hilbert spaces
 of functions on the domain $\Omega$, and show that the Barron space is a reproducing kernel Banach space with Lipschitz property,
 and every function in the graph  Barron space belongs to some graph RKHSs.
 In Section \ref{approximation.section}, we  consider approximation problem of functions in the Barron space by  outputs of  shallow GCNNs, and provide
 an estimate to number of parameters required for shallow GCNNs to reach a given approximation accuracy.
In Section \ref{complexity.section}, we consider Rademacher complexity of functions in the Barron space and show that
functions in the Barron space could be  well learnt from their random samples. In Section \ref{numericaldemon.section},
We use stochastic gradient descent with Nesterov momentum to train shallow GCNNs from both synthetic and real data, and we demonstrate its approximation performance in Sections \ref{approximation.section} and \ref{complexity.section}. All proofs are collected in Section \ref{proofs.section}.  In the Conclusion and Discussions section, we consider  a Barron space with convolution defined via the spatial approach and discuss its approximation property by the outputs of shallow GCNNs.

\bigskip

\section{Preliminaries}
\label{preliminaries.section}

 In this paper, we  consider
 weighted undirected connected graphs  $\mathcal{G}=({V}, E)$  of finite order $N\ge 1$ with their adjacency/degree/Laplacian matrices
 denoted by ${\bf W}, {\bf D}$ and ${\bf L}:={\bf D}-{\bf W}$ respectively,
and we define the
	{\em geodesic distance} $\rho(i,j)$   between vertices $i,j\in V$
	by  the number of edges in
the shortest path connecting them. 
For the convenience,  we may write the vertex set $V$ of the graph ${\mathcal G}$ as $\{1, \ldots, N\}$, and
use ${\bf x}=(x(i))_{i\in V}$  and also a vector ${\bf x}=[x(1), \ldots, x(N)]^T\in {\mathbb R}^N$
to denote a graph signal on the graph ${\mathcal G}$, where $x(i)$ is the signal value at the vertex $i$.

The concept of graph shifts is similar to the one-order delay in classical signal processing and
polynomial filters have been  widely used in graph signal processing.
In Section \ref{graphshifts.section}, we recall some preliminaries on commutative graph shifts and polynomial filters
\cite{ncjs22, Coutino17, Isufi2023, jiang19,    Ortegabook, Ortega18, aliaksei13, sandryhaila14,  Shuman2013, Stankovicchapter2019}.
   The graph
Fourier transform  is one of fundamental tools in graph signal processing
that decomposes graph signals  into different frequency components
and represents them by different modes of variation   \cite{Chen2023, Cheng2023, chung1997, Isufi2023,  Ortegabook, Ortega18,  Ricaud2019,  Shuman2013, Stankovicchapter2019}.
 Based on commutative  graph shifts,
in Section \ref{fourier.section} we introduce  graph  Fourier transform of a graph signal and  define graph convolution between two graph signals, see
\eqref{gft.def0} and \eqref{convolution.def}. We also show  that
the proposed  graph convolution operation
can be implemented in  the spectral domain  by taking the inverse Fourier transform of the multiplication between two Fourier transformed graph signals, and also in the spatial domain by applying some polynomial filtering procedure, see \eqref{convolution.def} and \eqref{convolution.def3}.
GCNNs   are
generalizations of classical CNNs to handle graph data, and they have been a powerful graph analysis method
\cite{Bronstein2017, Bruna2014, Defferrard2016, Dong2020,  Isufi2023, Kipf2017, Li2018, Liubook2020, Shi2019, Shuman2013, Wu2021, Zhang2019,  Zhou2020, Zhuang2018}.
In Section \ref{graphcnn.section}, we discuss the setting of a shallow  GCNN on a compact  domain of (sparse) graph signals.

 \subsection{Commutative graph shifts}
\label{graphshifts.section}

A {\em graph shift} ${\bf S}$, to be represented by a matrix ${\bf S}=(S(i,j))_{ i,j\in V}$,	is an 	elementary graph filter
satisfying
\begin{equation}
S(i,j)=0 \  \ {\rm if} \ \rho(i,j)\ge 2.
\end{equation}
 The illustrative examples of graph  shifts  
are the degree matrix  ${\bf D}$,
the  adjacency matrix  ${\bf W}$,
	the Laplacian matrix ${\bf L}$, the symmetric	normalized Laplacian matrix ${\bf L}^{\rm sym}={\bf D}^{-1/2} {\bf L}{\bf D}^{-1/2}$
  and their
	variants
\cite{ncjs22, Coutino17, Isufi2023,    Ortegabook, Ortega18, aliaksei13,  Shuman2013, Stankovicchapter2019}.
 A significant advantage of a graph shift ${\bf S}=(S(i,j))_{i,j\in V}$
is that the
 filtering procedure
 ${\mathbf S}:  (x(i))_{i\in V}=:{\bf x}\longmapsto {\bf S}{\bf x}:=(\tilde x(i))_{i\in V}$
 can be implemented by some local operation that updates  signal value  $\tilde x(i)$ at each vertex $i\in V$
	by a ``weighted" sum of signal values $x(j)$ at  {\em adjacent} vertices  $j\in {\mathcal N}_i$,
\begin{equation*}
\label{weightedsum}
\tilde{x}(i)=\sum_{j\in\mathcal{N}_i}S(i,j)x(j). 
\end{equation*}
where ${\mathcal N}_i$ is
the set of  adjacent vertices of $i\in V$.

Similar to  the  one-order delay $z_1^{-1}, \,\ldots\, ,  z_K^{-1}$ in classical multidimensional  signal processing,
the concept of  multiple commutative graph shifts ${\bf S}_1, \ldots,   {\bf S}_K$ is introduced in \cite{ncjs22},
where two illustrative  families of commutative graph shifts
 on circulant/Cayley graphs and product graphs are presented.
 Here we say that ${\bf S}_1, \ldots,    {\bf S}_K$ are {\em commutative} if
\begin{equation}\label{commutativityS}
	{\bf S}_k{\bf S}_{k'}={\bf S}_{k'}{\bf S}_k,\  1\le k,k'\le K.
\end{equation}
For commutative graph shifts, it is well known that  they can be upper-triangularized  simultaneously
by some unitary matrix  \cite[Theorem 2.3.3]{horn1990matrix}.
Under additional real-valued and symmetric assumptions, commutative graph shifts can be diagonalized
simultaneously
by some orthogonal matrix  ${\bf U}$, i.e.,
 \begin{equation}
	\label{diagonalization}
{\bf S}_k= {\bf U}
{\pmb \Lambda}_k{\bf U}^T
	\end{equation}
for some diagonal matrices
${\pmb \Lambda}_k:={\rm diag} (\lambda_k(n))_{1\le n\le N}, 1\le k\le K$.

Define
	\begin{equation}\label{jointspectrum.def} {\pmb \Lambda}=\big\{{\pmb \lambda}(n)=[\lambda_1(n), ..., \lambda_K(n)]^T, 1\le n\le N\big\}  \subset {\mathbb R}^K.\end{equation}
As $\lambda_k(n), 1\le n\le N$, are  eigenvalues of ${\bf S}_k,\ 1\le k\le K$,   we call $\pmb \Lambda$ as the {\em joint spectrum} of commutative graph shifts ${\bf S}_1, \,\ldots\, ,  {\bf S}_K$ \cite{ncjs22}.
Under the additional assumption that  ${\pmb \lambda}(n), 1\le n\le N$, in the joint spectrum ${\pmb \Lambda}$ are  distinct, it is shown in
\cite{ncjs22} that a filter ${\bf H}$ is a polynomial filter  if and only if it commutates with  ${\mathbf S}_1, \ldots, {\bf S}_K$, i.e.,
${\bf H} {\bf S}_k={\bf S}_k {\bf H}, \ 1\le k\le K$.
Here we say that  ${\bf H}$ is a {\em polynomial filter} of ${\mathbf S}_1, \ldots, {\bf S}_K$ if
\begin{equation}\label{MultiShiftPolynomial}
{\bf H}=h({\bf S}_1, \,\ldots\, ,  {\bf S}_K)=\sum h_{l_1,\ldots,l_K}{\bf S}_1^{l_1}\cdots {\bf S}_K^{l_K}
\end{equation}
for some  multivariate polynomial
$h(t_1, \,\ldots\, ,  t_K)=\sum 
h_{l_1,\ldots, l_K} \prod_{k=1}^K t_k^{l_k}$, where the sum is taken on the finite subset of ${\mathbb Z}_+^K$
\cite{ncjs22,  Isufi2023, jiang19, Ortegabook, Ortega18, aliaksei13, sandryhaila14,  Shuman2013, Stankovicchapter2019}.

\medskip

In this paper,  
 we always  make the following assumption on the graph shifts  ${\mathbf S}_1, \,\ldots\, ,  {\mathbf S}_K$ and
  their joint spectrum  ${\pmb \Lambda}$. 

\begin{assumption}\label{shiftassumption}
Graph shifts ${\mathbf S}_1, \,\ldots\, ,  {\mathbf S}_K$ are real-valued, symmetric and commutative, and
 ${\pmb \lambda}(n), 1\le n\le N$, in the joint spectrum ${\pmb \Lambda}$ in \eqref{jointspectrum.def} are  distinct.
\end{assumption}

  \subsection{Graph Fourier transform and graph convolution}
  \label{fourier.section}

  In this paper, we define the {\em  graph  Fourier transform} ${\mathcal F}{\bf x}$ of a graph signal ${\bf x}$
  and the {\em inverse  graph Fourier transform} ${\mathcal F}^{-1} {\pmb \omega}$ of a vector ${\pmb \omega}=[\omega(1), \,\ldots\, ,  \omega(N)]^T\in {\mathbb R}^N$  by
  \begin{equation}\label{gft.def0}
{\mathcal F} {\bf x}={\bf U}^T {\bf x}=[{\bf u}_1^T {\bf x}, \,\ldots\, ,  {\bf u}_N^T {\bf x}]^T \ \ {\rm and}\  \
 {\mathcal F}^{-1} {\pmb \omega}={\bf U} {\pmb \omega}=\sum_{n=1}^N \omega(n) {\bf u}_n,
  \end{equation}
  where
 ${\bf U}=[{\bf u}_1, \,\ldots\, ,  {\bf u}_N]$ is the orthogonal matrix in \eqref{diagonalization} to diagonalize commutative
  graph shifts ${\bf S}_1, \,\ldots\, ,  {\bf S}_K$ simultaneously.
  The conventional definition of the graph Fourier transform    on (un)directed graphs is based on one  graph shift and a common selection of the graph shift is either
  the Laplacian matrix ${\bf L}$ or the symmetric	normalized Laplacian matrix ${\bf L}^{\rm sym}$ on the graph
  \cite{Chen2023,  chung1997, ncjs22, Ortegabook,  Stankovicchapter2019}.
By \eqref{diagonalization}, the Parseval identity holds for the graph Fourier transform  ${\mathcal F}$ in \eqref{gft.def0},
\begin{equation} \label{gft.eq1}
\|{\mathcal F} {\bf x}\|_2= \|{\bf x}\|_2 \ \ {\rm for\ all}\  {\bf x}\in {\mathbb R}^N,
\end{equation}
and the operation of a polynomial filter ${\bf H}$ of  graph shifts ${\mathbf S}_k, 1\le k\le K$,  becomes a multiplier $m({\bf H})$ in the Fourier domain,
\begin{equation} \label{gft.eq2}
{\mathcal F} {\bf H} {\bf x}= m({\bf H})\odot ({\mathcal F} {\bf x}), \  {\bf x}\in {\mathbb R}^N,
\end{equation}
where ${\bf a}\odot {\bf b} $ is the Hadamard product of two vectors ${\bf a}$ and ${\bf b}\in {\mathbb R}^N$.
In particular, the multipliers associated with  the graph shifts ${\bf S}_k$ are the diagonal vectors of the  diagonal matrix ${\pmb \Lambda}_k, 1\le k\le K$.

  Without loss of generality,  we assume that ${\pmb\lambda}(n), 1\le n\le N$, have their norms $|{\pmb \lambda}(n)|:=\sum_{k=1}^K |\lambda_k(n)|$
in the nondecreasing order, i.e., $|{\pmb \lambda}(1)|\le |{\pmb \lambda}(2)|\le \,\ldots\, \le |{\pmb \lambda}(N)|$, otherwise we can perform certain permutation for
the orthogonal basis
${\bf u}_1, \ldots, {\bf u}_N$ of ${\mathbb R}^N$ to achieve the nondecreasing order.
 For  every $1\le n\le N$, we call
the vector  ${\pmb \lambda}(n)\in {\mathbb R}^K $, the graph signal ${\bf u}_n$ and  the quantity ${\bf u}_n^T {\bf x}$
 as  the $n$-th frequencies, the $n$-th fundamental frequency component and  the component of the signal ${\bf x}$ at the $n$-th frequency respectively.
 With the above ordering for the joint spectrum of commutative graph shifts,
 energy of smooth graph signals may concentrate
mainly at low frequencies
 \cite{Chen2023,
  chung1997, ncjs22,  Ortegabook, Stankovicchapter2019}.

\medskip

Given two graph signals ${\bf b}$ and ${\bf x}$, define their {\em  convolution} ${\bf b}*{\bf x}$ by
\begin{equation}\label{convolution.def}
{\bf b}*{\bf x}:={\mathcal F}^{-1} ({\mathcal F}{\bf b}\odot {\mathcal F}{\bf x})=
\sum_{n=1}^N  ({\bf u}_n^T {\bf b}){\bf u}_n {\bf u}_n^T {\bf x}.
\end{equation}
By \eqref{gft.eq2} and \eqref{convolution.def}, we see that the convolution associated with a graph signal ${\bf b}$ commutates with graph shifts ${\bf S}_k, 1\le k\le K$, i.e.,
$${\bf S}_k ({\bf b} * {\bf x})= ({\bf S}_k {\bf b})*{\bf x}, \  {\bf x}\in {\mathbb R}^N.$$
Therefore the convolution operation associated with a graph signal ${\bf b}$ could be written as a polynomial filtering procedure,
\begin{equation} \label{convolution.def3}
 {\bf b}*{\bf x}=h({\bf S}_1, \,\ldots\, ,  {\bf S}_K) {\bf x},\ {\bf x}\in {\mathbb R}^N, \end{equation}
where $h$ is a multivariate polynomial.
In particular,  we can show that \eqref{convolution.def3} holds  if and only if the  polynomial $h$  satisfies  the following interpolation property
 \begin{equation}\label{MultiShiftPolynomial.polynomial}
 h({\pmb \lambda}(n))= {\bf u}_n^T {\bf b},\  1\le n\le N,
 \end{equation}
 or equivalently
  \begin{equation}\label{convolution.def2}
{\bf b}=h({\bf S}_1, \,\ldots\, ,  {\bf S}_K) \sum_{n=1}^N {\bf u}_n= h({\bf S}_1, \,\ldots\, ,  {\bf S}_K) {\bf U} {\bf 1},
\end{equation}
where ${\bf 1}$ is  the column vector with all components taking value one.

\medskip

Let $0\le L\le N-1$. Denote the space of  all multivariate polynomials of degree at most $L$ by $\Pi_{L}$, and set
\begin{equation} \label{convolutionspace.def}
{\mathcal W}_{L}=\{{\bf b}=h({\bf S}_1, \,\ldots\, ,  {\bf S}_K) {\bf U} {\bf 1}: \  h\in \Pi_{L}\}.
\end{equation}
The spatial representation \eqref{convolution.def3} of the convolution operation
 provides another approach to implement  the convolution between graph signals ${\bf b}$ and ${\bf x}$ in the spatial domain. In particular, for graph signal ${\bf b}\in {\mathcal W}_{L}$, we
 first evaluate  the Fourier coefficients ${\bf u}_n^T {\bf b}, 1\le n\le N$, then find the multivariate polynomial $h$
  that take values ${\bf u}_n^T {\bf b}$ at the spectrum ${\pmb \lambda}(n), 1\le n\le N$; and finally used the distributed algorithm
to implement the polynomial filtering procedure in \eqref{convolution.def3}, see \cite{ncjs22}.  The total computational complexity to implement the distributed algorithm is
about  $ O(L^K \times N)$.

 Let
 ${\mathcal P}_{\pmb \Lambda}$ be the set of multivariate polynomials $p$ such that
 $p({\pmb \lambda}(n))=0, 1\le n\le N$.   One may verify that ${\mathcal P}_{\pmb \Lambda}$ is an ideal of the multivariate polynomial ring. Hence
there exists a Grobner basis $p_1, \ldots, p_I$ such that
for any polynomial $p\in {\mathcal P}_{\pmb \Lambda}$ there exist polynomial $q_1, \ldots, q_I$ such that
$p=q_1 p_1+\cdots q_I p_I$.   By \eqref{diagonalization}, we can show that polynomials $p_1, \ldots, p_I$ in the Grobner basis belong to $\Pi_{N-1}$.
On the other hand, it is shown in \cite[Theorem 1 in Chapter 9]{cheneybook}
that there exists a polynomial  $h_0\in \Pi_{N-1}$ satisfying the interpolation property  \eqref{convolution.def3}.
Therefore any polynomials $h_0+p, p\in {\mathcal P}_{\pmb \Lambda}$, satisfy  the interpolation property  \eqref{convolution.def3}
and the corresponding polynomial filter can be used to represent the convolution associated with a graph signal ${\bf b}$.
In general, due to fast distributed implementation in the spatial domain,  polynomials are selected to be of low degree, i.e., $h\in \Pi_{L}$ for some small $L\le N-1$.

\subsection{Graph convolution neural networks}
\label{graphcnn.section}

Let $\Omega\subset {\mathbb R}^N$ be a compact  set of graph signal inputs of GCNNs.
Due to the compactness of the set $\Omega$, there exists a positive constant $D_0$ such that
\begin{equation}\label{domain.def-1}
\Omega\subset \{ {\bf x}\in {\mathbb R}^N:\ \ \|{\bf x}\|_\infty \le  D_0\},
\end{equation}
 where we denote the standard  $\ell^p$-norm on
the linear space of $p$-summable graph signals  by $\|\cdot\|_p, 1\le p\le \infty$.
   An illustrative example of the domain $\Omega$ is
\begin{equation}\label{domain.def1}
\Omega_{\rm sp}=\{ {\bf x}\in {\mathbb R}^N:\ \  \|{\bf x}\| \le 1, \ \|{\bf x}\|_0\le s \},\end{equation}
the set of all $s$-sparse graph signals bounded by one, where
 $\|\cdot\|$ is a norm on ${\mathbb R}^N$ (such as the standard  $\ell^p$-norm $\|\cdot\|_p, 1\le p\le \infty$), and
  $\|{\bf x}\|_0$ is the number of nonzero entries of the vector ${\bf x}$.
 Taking $s=N$ leads to
the unit ball widely-used in GCNNs,
\begin{equation}\label{domain.def0}
\Omega_{\rm ba}=\{ {\bf x}\in {\mathbb R}^N:\ \ \|{\bf x}\| \le 1 \}\end{equation}
 with one as its radius  and the origin as its  center.
 In the classical  neural network setting, a  popular selection of the domain $\Omega$ is  the unit cube $[0,1]^N$ \cite{Barron1993, Bruna2014, EMaWu2022}.

\medskip

Let $\|\cdot\|$ be a norm on ${\mathbb R}^N$ so that the ReLU activation function $\sigma$ in GCNNs satisfy
\begin{equation}\label{connorm0.def}
\|\sigma({\bf x})\|\le \|{\bf x}\|  \ \ {\rm for \ all}\  {\bf x}\in {\mathbb R}^N.
\end{equation}
Denote its dual norm by $\|\cdot\|_*$.
Our illustrative examples of the norm $\|\cdot\|$ and its dual norm $\|\cdot\|_*$ are the $\ell^p$-norm $\|\cdot\|_p$ and its dual norm $\|\cdot\|_q$,
  where
$1/p+1/q=1$.
Due to the norm equivalence on a finite-dimensional linear space, one may verify that the ReLU activation function $\sigma$  has Lipschitz property on $\Omega$ with Lipschitz constant denoted by $\|\sigma\|_{\rm Lip}$,
\begin{equation} \label{lipschitz.eq2}
\|\sigma({\bf x})-\sigma({\bf x}')\|\le
\|\sigma\|_{\rm Lip}\|{\bf x}-{\bf x}'\| \ \ {\rm for \ all} \ {\bf x},{\bf x}'\in {\mathbb R}^N.\end{equation}

Let ${\mathcal W}$ be a linear space of graph signals used for the convolution in GCNNs.  Our representative example
is the space ${\mathcal W}_{L}, 1\le L\le N-1$, in \eqref{convolutionspace.def}, and hence the  graph convolution associated with a graph signal in  ${\mathcal W}_{L}$ can be implemented by the polynomial filtering procedure
in  a distributed manner.  In the standard CNN setting, a  popular selection of  graph convolutions is the family of $N\times N$ symmetric Toeplitz matrices with bandwidth $L$, where the shifting structure of Toeplitz matrices can be described by the circular graph.

For a graph signal ${\bf b}\in {\mathcal W}$, define a {\em convolution norm}
$\|{\bf b}\|_{\rm co}$ such that
\begin{equation}\label{connorm.def}
\|{\bf b}*{\bf x} \|\le \|{\bf b}\|_{\rm co} \ {\rm for \ all}\ {\bf x}\in \Omega.
\end{equation}
To consider the Lipschitz property of functions in the graph Barron space  in Corollary \ref{lipschitz.cor}
and uniform approximation property in Theorem \ref{inverseapproximaton.thm}, we also require that the convolution norm satisfies
the Lipschtiz property with Lipschitz bounded by a multiple of the convolution norm,
\begin{equation}\label{connorm.def1}
\|{\bf b}*({\bf x}-{\bf x}') \|\le  D_1\|{\bf b}\|_{\rm co} \|{\bf x}-{\bf x}'\| \ {\rm for \ all}\ {\bf x}, {\bf x}'\in \Omega
\end{equation}
where $D_1$ is a positive constant.
To consider the Rademacher complexity in Theorem \ref{Rademacher.thm}, we need  a stronger requirement than the Lipschitz property \eqref{connorm.def1}: there exists a positive constant $D_2$ such that
\begin{equation}\label{connorm.def2}
\Big\|{\bf b}*\Big(\sum_{i=1}^S \epsilon_i {\bf x}_i\Big) \Big\|\le  D_2\|{\bf b}\|_{\rm co}  \Big\|\sum_{i=1}^S \epsilon_i {\bf x}_i\Big\|
\end{equation}
hold for all  $\epsilon_i\in \{-1, 1\}$ and  ${\bf x}_i\in \Omega, 1\le i\le S$ and all $S\ge 1$.
An illustrative example of the convolution norm $\|{\bf b}\|_{\rm co}$ of a graph signal ${\bf b}\in {\mathcal W}$
is
$$\|{\bf b}\|_{\rm co}=D_3 \|{\bf b}\|_{\rm coop}$$
where the constant $D_3$  is chosen so that $\|{\bf x}\|\le D_3$ for all ${\bf x}\in \Omega$, and
$$\|{\bf b}\|_{\rm coop}=\sup_{{\bf 0}\ne {\bf x}\in {\mathbb R}^N} {\|{\bf b}*{\bf x}\|}/{\|{\bf x}\|}$$
is the operator norm of the convolution ${\bf b}*$.
For the above setting,  the constants $D_1$ in \eqref{connorm.def1} and  $D_2$  in  \eqref{connorm.def2} are given
by  $D_1=D_2=D_3$.

If the convolution associated with the graph signal ${\bf b}$ in ${\mathcal W}_{L}$ can be represented by a polynomial filter  $h({\mathbf S}_1, \ldots, {\mathbf S}_K)$ in \eqref{convolution.def2}, we may define the convolution norm by an appropriate scaling of
$$\|{\bf b}\|_{\rm cofi}=\sum  |h_{l_1,\ldots, l_K}| \prod_{k=1}^K \|{\bf S}_k\|^{l_k},$$
where  $h(t_1, \,\ldots\, ,  t_K)=\sum  h_{l_1,\ldots, l_K} \prod_{k=1}^K t_k^{l_k}$ and $\|{\bf S}_k\|$ is the operator norm of graph shifts ${\bf S}_k, 1\le k\le K$.

\medskip

In this paper, we consider  shallow GCNNs with
 parameters  ${\pmb \Theta}=({\pmb \theta}_1, \ldots, {\pmb \theta}_M)\in ({\mathbb R}^N\times {\mathcal W}\times {\mathbb R}^N)^M$, which have graph signals ${\bf x}\in \Omega$ as
  inputs and
$f_M({\bf x}, {\pmb \Theta})$ in \eqref{fM.def}
as outputs.

\medskip

Barron space of functions on the unit cube $[0, 1]^N$ was introduced in \cite{Barron1993}, where
it is shown that functions in Barron space are well approximated by  the classical shallow neuron networks.
In this paper,  we
introduce Barron space  ${\mathcal B}$ of  functions $f$ of graph signals ${\bf x}\in \Omega$, 
and discuss its approximation property by some  shallow GCNNs, i.e.,
$$ f({\bf x})\approx  f_M({\bf x}, {\pmb \Theta}),\  {\bf x}\in \Omega$$
for some ${\pmb \Theta}\in ({\mathbb R}^N\times {\mathcal W}\times {\mathbb R}^N)^M$, see Sections \ref{barronspace.section}, \ref{approximation.section} and \ref{complexity.section}
 for  theoretical results  and Section \ref{numericaldemon.section} for numerical demonstrations.

\section{Barron space on graphs}
\label{barronspace.section}

Let  ${\mathcal G}=(V, E)$ be a weighted undirected graph of  order $N\ge 1$,
graph shifts
${\bf S}_1, \ldots , {\bf S}_K$ on the graph ${\mathcal G}$  satisfy Assumption \ref{shiftassumption},
 $\Omega$ be the domain for graph signal inputs of GCNNs,   ${\mathcal W}$ be the linear space of graph signals used for the convolution
 in GCNNs, and $\sigma$ be the  ReLU activation function in  GCNNs.

 Barron space of functions on the unit cube $[0, 1]^N$ was introduced in \cite{Barron1993} with the help of Fourier transform.
In  \cite{Bach2017},  Bach considered the space ${\mathcal F}_1$ of functions $f$ with the following  spatial representation
\begin{equation}\label{Spatialrepresentation.def}
f({\bf x}) = \int_{\mathcal V} \phi_{z}({\bf x}) \rho(dz),\ {\bf x}\in \Omega,\end{equation}
  where $\phi_z, z\in  {\mathcal V}$,
is a family of basis functions (a.k.a neurons) and  $\rho$ is a signed Radon measure on ${\mathcal V}$ with finite total variation  $|\rho|({\mathcal V})$.
In \cite{EMaWu2022, E2022}, E, Ma and Wu introduced a  Barron space of functions in \eqref{Spatialrepresentation.def} with $\rho$ being a probability measure and
  $\phi_{z} ({\bf x})=u \sigma({\bf v}^T{\bf x}+w)$ being neurons
 where  $z={(u, {\bf v}, w)}\in {\mathbb R}\times {\mathbb R}^N\times {\mathbb R}$.
 In this paper, we consider functions $f:\Omega\to {\mathbb R}$ on the  domain $\Omega$ that
  can be written as
\begin{equation}\label{functionrepresentation.def}
f({\bf x}) = \int_{\R^N\times {\mathcal W}  \times \R^N} {\bf a}^T \sigma \rb{{\bf  b} * {\bf x} + {\bf c} } \rho\rb{d{\bf a},d{\bf b},d{\bf c}}, \ {\bf x}\in \Omega
\end{equation}
 where  $\rho$ is a probability measure on $\R^N\times {\mathcal W}  \times \R^N$.
 We remark that functions in \eqref{functionrepresentation.def} have the spatial representation of the form
 \eqref{Spatialrepresentation.def} with neurons $\phi_{\pmb \theta} ({\bf x})={\bf a}^T \sigma({\bf b}*{\bf x}+{\bf c})$ of GCNNs, where
 ${\pmb \theta}=({\bf a}, {\bf b}, {\bf c})\in  \R^N\times {\mathcal W}  \times \R^N$.
  In this section, we introduce a graph Barron space  ${\mathcal B}$
 of functions  with the spatial representation \eqref{functionrepresentation.def}, see  \eqref{barronnorm.def} and \eqref{Barron.def00}.

Reproducing kernel Banach spaces have been frequently used in neural networks,   machine learning, sampling theory,
sparse approximation and functional analysis
 \cite{Bartoluccci2023, Fasshauer2015, Lin2022, nashed2010, Song2013, Xu2019, Zhang2013,  Zhang2009}.
 In  Section \ref{Barron.subsection},  we show that the graph Barron space  ${\mathcal B}$ is a reproducing kernel Banach space
  and has Lipschitz property, see Theorem \ref{BanachBarronspace.thm} and Corollary
 \ref{lipschitz.cor}.

 Reproducing kernel Hilbert spaces  (RKHSs) have been widely accepted in kernel-based learning for  function estimation with dimension independent error, and
their kernels are usually selected to measure certain similarity between  input data  that could significantly save computation costs  \cite{Ghojogh2021, Rahimi2007, Scholkopf2002, Steinwart2008}.  In Section \ref{rkhs.subsection}, we introduce a family of RKHSs with neuron kernels, see \eqref{rkhskernel.def}.  Also we  provide a representation theorem for functions in the  RKHS,
  and show that every function in the graph Barron space belongs to some RKHSs, see Theorem \ref{rkhs.thm} and   \ref{BarronRkhs.thm}.

\subsection{Barron spaces and  reproducing kernel Banach spaces}
\label{Barron.subsection}

Let the norm $\|\cdot\|$, its dual norm $\|\cdot\|_*$ and the convolution norm $\|\cdot\|_{\rm co}$ be as in Section
\ref{graphcnn.section}.
For $1\le r\le \infty$, let  ${\mathcal B}_r$ contain  all functions $f$ on the domain $\Omega$ with the spatial representation \eqref{functionrepresentation.def}
such that
${\mathbb E}_\rho\big [\|{\bf a}\|_*^r (\|{\bf b}\|_{\rm co}+ \|{\bf c}\|)^r\big]<\infty
$
if $1\le r<\infty$, and
the support of the  probability measure  $\rho$
being bounded if $r=\infty$.
Define the norm $\|f\|_{{\mathcal B}_r}$ of a function $f\in {\mathcal B}_r$ by
\begin{equation}\label{barronnorm.def}
    \norm{f}_{\B_r} = \begin{cases}
    \displaystyle \inf_{\rho \in \mathcal{P}_f} \big[ \E_\rho \norm{\bf a}_*^r(\norm{\bf b}_{\rm co} +\norm{\bf c})^r \big] ^{1/r} & {\rm if}\  1\le r<\infty\\
    \displaystyle \inf_{\rho \in \mathcal{P}_f} \max_{\rb{{\bf a},{\bf b},{\bf c}}\in \text{supp}\rb{\rho}} \norm{\bf a}_*
    ({\norm{\bf b}_{\rm co} +\norm{\bf c}}) & {\rm if} \ r=\infty,
    \end{cases}
\end{equation}
 where ${\mathcal P}_f$ is  the
collection of  probability measure $\rho$ in the representation \eqref{functionrepresentation.def}.

In the following theorem,  we  show that
the  normed vector space ${\mathcal B}_r$ for GCNNs are reproducing kernel Banach spaces independent on $1\le r\le \infty$,  see Section \ref{BanachBarronspace.thm.pfsection} for the proof.

\begin{theorem}\label{BanachBarronspace.thm} Let ${\mathcal B}_r, 1\le r\le \infty$, be as in \eqref{barronnorm.def}.
Then they  are the same reproducing kernel Banach space. Moreover,
\begin{equation}\label{BanachBarronspace.eq1}
|f({\bf x})|\le \|f\|_{{\mathcal B}_r}, \ {\bf x}\in \Omega
\end{equation}
and
\begin{equation} \label{BanachBarronspace.eq2}
\|f\|_{{\mathcal B}_\infty}=
\|f\|_{{\mathcal B}_1}=\|f\|_{{\mathcal B}_r}
\end{equation}
hold for all $f\in {\mathcal B}_r, 1\le r\le \infty$.
\end{theorem}

By Theorem \ref{BanachBarronspace.thm}, we  denote
the reproducing kernel Banach spaces ${\mathcal B}_r, 1\le r\le \infty$,  by ${\mathcal B}$ and define its norm  by $\|\cdot\|_{\mathcal B}$, i.e.,
\begin{equation}\label{Barron.def00}
{\mathcal B}={\mathcal B}_r\ \ {\rm and} \ \ \|\cdot\|_{\mathcal B}=\|\cdot\|_{{\mathcal B}_r}, \ 1\le r\le \infty.
\end{equation}
Following the terminology in \cite{EMaWu2022}, we call  the reproducing kernel Banach space ${\mathcal B}$ as the {\em graph Barron space}.

Let
\begin{equation}\label{ST.def}{\mathbb S} = \{{\bf a}\in {\mathbb R}^N:\, \|{\bf a}\|_*=1 \}\ \
{\rm and} \ \  {\mathbb T} = \{ ({\bf b}, {\bf c})\in {\mathcal W}\times {\mathbb R}^N: \, \|{\bf b}\|_{\rm co}+\|{\bf c}\|=1 \}\end{equation}
be the unit spheres in ${\mathbb R}^N$ and ${\mathcal W}\times {\mathbb R}^N$ respectively. A crucial step in the proof of Theorem \ref{BanachBarronspace.thm} is the following
spatial representation of  functions  $f$ in the graph Barron space  ${\mathcal B}$
for some  probability measure $\hat \rho$ on ${\mathbb S}\times {\mathbb T}$,
\begin{equation} \label{functionrepresentation.def2}
f({\bf x})=\|f\|_{{\mathcal B}} \int_{{\mathbb S}\times {\mathbb T}} {\bf a}^T \sigma({\bf b}*{\bf x}+{\bf c}) \hat \rho(d{\bf a}, d{\bf b}, d{\bf c}),
\end{equation}
see Lemma \ref{quasinorm2.lem}.

Under the additional assumption that the  ReLU function $\sigma$ and the convolution norm satisfy
\eqref{lipschitz.eq2} and
\eqref{connorm.def1} respectively,
 for any $f\in {\mathcal B}$ and ${\bf x}, {\bf x}'\in \Omega$, we obtain from \eqref{functionrepresentation.def2} that
\begin{eqnarray}\label{lipschitz.eq1}
 |f({\bf x})- f({\bf x}')|
 & \hskip-0.08in \le & \hskip-0.08in
\|f\|_{{\mathcal B}}\int_{{\mathbb S}\times {\mathbb T}}
\|{\bf a}\|_* \|\sigma({\bf b}*{\bf x}+{\bf c})- \sigma({\bf b}*{\bf x}'+{\bf c})\|\hat \rho(d{\bf a}, d{\bf b}, d{\bf c})\nonumber\\
 & \hskip-0.08in \le & \hskip-0.08in
\|f\|_{{\mathcal B}}\int_{{\mathbb S}\times {\mathbb T}}  \|\sigma\|_{{\rm Lip}}
\| {\bf b}*({\bf x}-{\bf x}')\| \hat \rho(d{\bf a}, d{\bf b}, d{\bf c})\nonumber\\
\quad  & \hskip-0.08in \le & \hskip-0.08in D_1 \|\sigma\|_{{\rm Lip}}\|f\|_{{\mathcal B}}\|{\bf x}-{\bf x'}\|,\ \ {\bf x}, {\bf x}'\in \Omega.
 \end{eqnarray}
 Therefore  functions in the graph Barron space ${\mathcal B}$ have Lipschitz property, cf. \cite{Barron1993} and \cite[Theorem 3.3]{E2022}.

\begin{corollary}  \label{lipschitz.cor}
Let ${\mathcal B}$ be the graph Barron space of functions on the domain $\Omega$ given in \eqref{Barron.def00}.
If  the ReLU activation function $\sigma$ satisfies \eqref{connorm0.def} and \eqref{lipschitz.eq2},  and
if  the convolution norm  $\|\cdot\|_{\rm co}$ satisfies
\eqref{connorm.def} and
\eqref{connorm.def1}, then  any function $f$ in the graph Barron space ${\mathcal B}$ has the Lipschitz property with Lispchitz constant bounded by $D_1 \|\sigma\|_{{\rm Lip}}\|f\|_{\mathcal B}$, where  $\|\sigma\|_{{\rm Lip}}$ and $D_1$
are the constants in  \eqref{lipschitz.eq2}  and  \eqref{connorm.def1} respectively.
\end{corollary}

\subsection{Reproducing kernel Hilbert spaces with neuron kernels}
\label{rkhs.subsection}

Let $\widehat P$ be the set of all probability measures on ${\mathbb S}\times {\mathbb T}$.
Given $\hat \rho\in \widehat P$,  we
 define a  reproducing kernel Hilbert spaces (RKHS for abbreviation)
${\mathcal H}_{\hat \rho}$ of functions on the  domain $\Omega$, whose
 kernel function $K_{\hat \rho}$ is defined by
\begin{equation}\label{rkhskernel.def}
K_{\hat \rho}({\bf x}, {\bf x}')=\int_{{\mathbb S}\times {\mathbb T}}  {\bf a}^T \sigma({\bf b}*{\bf x}+{\bf c})  {\bf a}^T \sigma({\bf b}*{\bf x}'+{\bf c})\hat \rho(d{\bf a}, d{\bf b}, d{\bf c}), \
   {\bf x},  {\bf x}'\in \Omega.
\end{equation}
One may verify that the RKHS ${\mathcal H}_{\hat \rho}$ is the completion of the inner product space ${\mathcal H}_{\hat \rho}^o$, where
${\mathcal H}_{\hat \rho}^o$ is the  linear span  of
$K_{\hat \rho}(\cdot, {\bf x}'), {\bf x'}\in \Omega$, and
the inner  product on ${\mathcal H}_{\hat \rho}^o$
between $g_1= \sum _{j=1}^{J}b_{j}K_{\hat \rho}(\cdot, {\bf x}_{j})$ and $g_2= \sum _{i=1}^{I}a_{i}K_{\hat \rho}(\cdot, {\bf x}_{i}^\prime)\in {\mathcal H}_{\hat\rho}^o$
 is defined by
\begin{equation*}\label{rkhs.def2}
\langle g_1, g_2\rangle_{{\mathcal H}_{\hat \rho}}=  
\sum _{i=1}^{I}\sum _{j=1}^{J}{a_{i}}b_{j}K_{\hat \rho}({\bf x}_{j},{\bf x}_{i}^\prime).\end{equation*}

Let $L^p_{\hat \rho}({\mathbb S}\times {\mathbb T}), 1\le p<\infty$, be the Banach space of
all $p$-integrable functions on ${\mathbb S}\times {\mathbb T}$ with respect to the probability measure $\hat \rho$ and  define its norm by
$$\|\eta\|_{L^p_{\hat \rho}}= \Big(\int_{{\mathbb S}\times {\mathbb T}} |  \eta({\bf a}, {\bf b}, {\bf c})|^p \hat \rho(d{\bf a}, d{\bf b}, d{\bf c})\Big)^{1/p}.$$
Denote   the completion of the linear space spanned by $ {\bf a}^T \sigma({\bf b}*{\bf x}+{\bf c}), {\bf x}\in \Omega$,
in $L^2_{\hat \rho}({\mathbb S}\times {\mathbb T})$ by ${\mathcal L}_{\hat \rho}^2$, and the orthogonal projection from  $L^2_{\hat \rho}({\mathbb S}\times {\mathbb T})$
 onto  its Hilbert subspace ${\mathcal L}_{\hat \rho}^2$ by $P_{\hat \rho}$.  In the following theorem, we show that
a function $f$ in
the  RKHS ${\mathcal H}_{\hat \rho}$ can be represented by some function $\eta\in L^2_{\hat \rho}({\mathbb S}\times {\mathbb T})$, see Section \ref{rkhs.thm.pfsection} for the proof.

\begin{theorem}\label{rkhs.thm}
 Let $\hat\rho$ be a probability measure on ${\mathbb S}\times {\mathbb T}$ and ${\mathcal H}_{\hat \rho}$ be the  RKHS
 with kernel $K_{\hat \rho}$ given in \eqref{rkhskernel.def} and norm denoted by $\|\cdot\|_{{\mathcal H}_{\hat \rho}}$.
 Then  $g\in {\mathcal H}_{\hat \rho}$ if and only if
 \begin{equation}\label{rkhs.thm.eq1}
 g({\bf x})=\int_{{\mathbb S}\times {\mathbb T}}  {\bf a}^T \sigma({\bf b}*{\bf x}+{\bf c})\eta({\bf a}, {\bf b}, {\bf c}) \hat\rho(d{\bf a}, d{\bf b}, d{\bf c})
 \end{equation}
 for some  function $\eta\in L^2_{\hat \rho}({\mathbb S}\times {\mathbb T})$. Moreover,
 \begin{equation} \label{rkhs.thm.eq2}
 \|g\|_{{\mathcal H}_{\hat \rho}}= \|P_{\hat\rho} \eta\|_{L^2_{\hat \rho}}
 \end{equation}
 and
\begin{equation}\label{rkhs.thm.eq3}
|g({\bf x})|\le \|g\|_{{\mathcal H}_{\hat \rho}}, \ {\bf x}\in \Omega.
\end{equation}
\end{theorem}

In the  neuron network setting,  spaces  ${\mathcal F}_1$ and ${\mathcal F}_2$  are introduced in \cite{Bach2017}, where functions in
${\mathcal F}_1$ and ${\mathcal F}_2$  have similar representations to the one in \eqref{rkhs.thm.eq1}
with $\eta$ being integrable and square-integrable on some compact set respectively, cf. \cite{EMaWu2022}.

\begin{remark} \label{rkhs.rem1}
{\rm
Let $\psi\big({\bf x}, ({\bf a}, {\bf b}, {\bf c})\big)= {\bf a}^T \sigma({\bf b}*{\bf x}+{\bf c})\in {\mathcal L}^2_{\hat \rho}$
and $\langle \cdot, \cdot\rangle_{\hat \rho}$
 be the inner product on $L^2_{\hat \rho}({\mathbb S}\times {\mathbb T})$.
 We remark that for a function $g$ in the RKHS ${\mathcal H}_{\hat \rho}$,  the representing function $\eta$
 in
 \eqref{rkhs.thm.eq1} is not unique.
  In particular, it can be replaced by another representation function $\tilde \eta$, such as $\tilde \eta=P_{\hat \rho} \eta \in {\mathcal L}^2_{\hat \rho}$, so that $\eta-\tilde \eta$ is orthogonal to ${\mathcal L}^2_{\hat \rho}$,  because
 for any ${\bf x}_0\in \Omega$,
$$
 g({\bf x}_0)- \int_{{\mathbb S}\times {\mathbb T}}  {\bf a}^T \sigma({\bf b}*{\bf x}_0+{\bf c}) \tilde \eta({\bf a}, {\bf b}, {\bf c}) \hat\rho(d{\bf a}, d{\bf b}, d{\bf c})
 =  \langle \eta-\tilde \eta, \psi({\bf x}_0, \cdot)\rangle_{\hat \rho} =0, $$
 where the last equality holds as $\psi({\bf x}_0, \cdot)\in {\mathcal L}^2_{\hat \rho}$ for all ${\bf x}_0\in \Omega$.
  Denote the set of such representing functions in $L^2_{\hat \rho}({\mathbb S}\times {\mathbb T})$  by $\widehat {\mathcal P}_{g}$
  and the orthogonal complement of  ${\mathcal L}^2_{\hat \rho}$ in
  $L^2_{\hat\rho}({\mathbb S}\times {\mathbb T})$ by $({\mathcal L}_{\hat\rho}^2)^\perp$. Then
  \begin{equation}  \label{rkhs.thm.projectioneq}
  \widehat {\mathcal P}_g=\eta+ ({\mathcal L}_{\hat\rho}^2)^\perp
  \end{equation}
  for every $\eta\in \widehat {\mathcal P}_g$, and
  \begin{equation} \label{rkhs.thm.projectioneq+}
  \|g\|_{{\mathcal H}_{\hat \rho}}= \inf_{\eta\in \widehat P_g} \|\eta\|_{L^2_{\hat \rho}}.
  \end{equation}
  }\end{remark}

\begin{remark}\label{rkhs.rem2}
{\rm We remark that representing functions  $\eta ({\bf a}, {\bf b}, {\bf c})\in {\mathcal L}_{\hat \rho}^2$ for the RKHS ${\mathcal H}_{\hat \rho}$ is linear with respect to ${\bf a}=[a_1, \ldots, a_N]^T$, i.e.,
\begin{equation}\label{rkhs.rem.def1}
\eta ({\bf a}, {\bf b}, {\bf c})=\sum_{n=1}^N a_n \tau_n({\bf b}, {\bf c})\end{equation}
for some functions  $\tau_n, 1\le n\le N$, on ${\mathbb T}$.  Let ${\bf e}_n\in {\mathbb R}^N, 1\le n\le N$, be the unit vector with zero components except its $n$-th component taking value one. Observe that
$$ |{\bf e}_n^T ({\bf b}*{\bf x})|\le \|{\bf e}_n\|_* \|{\bf b}*{\bf x}\|\le \|{\bf e}_n\|_* \|{\bf b}\|_{\rm co}, 1\le n\le N\ {\rm and}\ {\bf x}\in \Omega.$$
Therefore in addition to the linearity about ${\bf a}$ for representing functions in the RKHS ${\mathcal H}_{\hat\rho}$,  the functions $\tau_n, 1\le n\le N$, in \eqref{rkhs.rem.def1} are linear with respect to ${\bf b}$ and ${\bf c}$ in the domain
$\{({\bf b}, {\bf c})\in {\mathbb T}: \ \|{\bf e}_{n'}\|_* \|{\bf b}\|_{\rm co}\le | {\bf e}_{n'}^T {\bf c}|, 1\le n'\le N\}.$
 }\end{remark}

In the following theorem, we show that  RKHSs  ${\mathcal H}_{\hat \rho}, \hat \rho\in \hat P$,  for GCNNs
are closely related to the graph Barron space  ${\mathcal B}$
in
\eqref{Barron.def00}, see Section \ref{BarronRkhs.thm.pfsection} for the proof.

\begin{theorem}\label{BarronRkhs.thm}
 Let ${\mathcal B}$ be the graph Barron space in
\eqref{Barron.def00}, and
${\mathcal H}_{\hat \rho}, \hat\rho\in \widehat{\mathcal P}$, be RKHSs with  kernels given in  \eqref{rkhskernel.def}.  Then
\begin{equation}\label{BarronRkhs.thm.eq1}
{\mathcal B}= \cup_{{\hat \rho}\in \widehat P} {\mathcal H}_{\hat \rho}\end{equation}
and
\begin{equation} \label{BarronRkhs.thm.eq2}
\|f\|_{\mathcal B}=\inf _{f\in H_{\hat \rho},\ \hat \rho\in \hat P} \|f\|_{{\mathcal H}_{\hat \rho}}, \ f\in {\mathcal B}.
\end{equation}
\end{theorem}

In the standard neuron network setting, a similar conclusion to the one in \eqref{BarronRkhs.thm.eq1}
about RKHSs and the Barron space is given in \cite[Proposition 3]{EMaWu2022}, however  the corresponding norm equivalence
in \eqref{BarronRkhs.thm.eq2} is not mentioned.

\section{Approximation theorems on graph Barron spaces}
\label{approximation.section}

Given a shallow GCNN  with parameter  ${\pmb \Theta}=({\pmb \theta}_1, \ldots, {\pmb \theta}_M)\in ({\mathbb R}^N\times {\mathcal W}\times {\mathbb R}^N)^M$,
we define  its {\em $p$-path norm}   by
\begin{equation}\label{pathnorm.def}
    \|{\pmb \Theta}\|_{P, p} = \left\{ \begin{array}{ll} \big( M^{-1} \sum_{ m=1}^M  \|{\pmb \theta}_m\|^p \big)^{1/p}  & {\rm if} \ 1\le p<\infty\\
     \sup_{1\le m\le M}  \|{\pmb \theta}_m\| & {\rm if} \ p=\infty,
     \end{array}\right.
\end{equation}
where $\|{\pmb \theta}\|=\|{\bf a}\|_* (\|{\bf b}\|_{\rm co}+ \|{\bf c}\|)$ for $\pmb \theta=({\bf a}, {\bf b}, {\bf c})\in {\mathbb R}^N \times {\mathcal W}\times {\mathbb R}^N$, cf. \cite{EMaWu2022} for $p=1$.  One may verify that the output of the shallow GCNN with parameter ${\pmb \Theta}$
belongs to the Barron space ${\mathcal B}$,
\begin{equation}\label{fM.eq1}
\Big\|\frac{1}{M} \sum_{m=1}^M \phi({\bf x}, {\pmb \theta}_m)\Big\|_{\mathcal B}\le \frac{1}{M} \sum_{m=1}^M \|{\pmb \theta}_m\|\le \|{\pmb \Theta}\|_{P, p}, 1 \le p\le \infty,
\end{equation}
where $\phi({\bf x}, {\pmb \theta}_m)={\bf a}_m^T \sigma({\bf b}_m*{\bf x}+{\bf c}_m)$ for ${\pmb \theta}_m=({\bf a}_m, {\bf b}_m, {\bf c}_m), 1\le m\le M$.
In the classical neuron network setting, functions in  Barron/Besov/H\"older spaces can be well approximated by
outputs of some neural networks, see 
\cite{EMaWu2022, Shen2022, Siegel2023, Yang2023a, Yang2023b} and references therein.
In Section \ref{approximation.subsection},  we show that  functions  in the graph Barron space can be approximated
in integrated square norm and
uniform norm by outputs of some shallow GCNNs with bounded path norm, see Theorems
\ref{approximation.thm} and  \ref{approximation.thm2}.  As a consequence, we conclude that integrated square error to approximate
functions in the Barron space can be achieved by some shallow GCNNs with the number of parameters being linear about the graph size and inverse square of the approximation error.

Let $Q\ge 0$ and $1\le p\le \infty$. Consider the set of  outputs of  all shallow GCNNs with $p$-path norms
of their parameters bounded by $Q$,
\begin{equation}\label{NQ.def}
{\mathcal C}_{Q, p}=\cup_{M=1}^\infty {\mathcal C}_{Q, p, M},
\end{equation}
where
\begin{equation}\label{NQM.def}
    \mathcal{C}_{Q, p,  M} = \Big\{ \frac{1}{M} \sum_{m=1}^M \phi({\bf x}, {\pmb \theta}_m):\, \big\|({\pmb \theta}_1, \ldots, {\pmb \theta}_M)\big\|_{P, p}\le Q\Big\}, \ M\ge 1.
    \end{equation}
    By \eqref{fM.eq1}, we see that any function $g\in {\mathcal C}_{Q,p}$ belongs to the Barron space ${\mathcal B}$  and has its Barron norm bounded by $Q$,  i.e.,
\begin{equation}
\label{NQBarron.eq}
\|g\|_{\mathcal B}\le Q \ \ {\rm for \ all}\ g\in {\mathcal C}_{Q, p}.
\end{equation}
In Theorems \ref{inverseapproximaton.thm} and \ref{density.thm} of
Section \ref{inverseapproximation.subsection}, we
show that the limit of  functions in ${\mathcal C}_{Q, p}$ belongs to the graph Barron space ${\mathcal B}$,
and the  set  $\cup_{Q\ge 0} {\mathcal C}_{Q, p}$ could be  dense in the space of continuous functions on the domain.
Universal approximation theorem is one of fundamental problems in theoretical learning  research
\cite{Bruel2020, kerven2021, kerven2019}. We remark that the conclusion in Theorem \ref{density.thm} can be considered as a universal approximation theorem for a shallow  GCNN, c.f. Corollary \ref{univeral.cor}.

\subsection{Approximation theorems}
\label{approximation.subsection}
First we show that
  shallow GCNNs may approximate any function  
 in the graph Barron space  ${\mathcal B}$ in integrated square norm, see Section \ref{approximation.thm.pfsection} for the proof. 

\begin{theorem}\label{approximation.thm}
Let $M\ge 1$,  
$f\in {\mathcal B}$ and $\mu$ be a probability measure on the domain $\Omega$. Then for any $\epsilon>0$
there is a shallow GCNN  with parameter  ${\pmb \Theta}=({\pmb \theta}_1, \ldots, {\pmb \theta}_M)$
such that
\begin{equation}\label{approximation.thm.eq1}
\|{\pmb \Theta}\|_{P, \infty}\le \|f\|_{{\mathcal B}}
\end{equation}
and
\begin{equation}\label{approximation.thm.eq2}
    \int_{\Omega}\Big |\frac{1}{M} \sum_{m=1}^M \phi({\bf x}, {\pmb \theta}_m)-f({\bf x})\Big|^2  \mu(d{\bf x}) \le \frac{1+\epsilon}{M} {\norm{f}_{\mathcal B}^2}.
\end{equation}
\end{theorem}

Taking Dirac measure at some ${\bf x}_0\in \Omega$ as the probability measure  in Theorem
\ref {approximation.thm}, we have the following pointwise estimate.

\begin{corollary}\label{approximation.cor}
Let $M\ge 1$ and $f\in {\mathcal B}$. Then for any $\epsilon>0$ and ${\bf x}_0\in \Omega$,
there is a shallow GCNN  with parameter  ${\pmb \Theta}=({\pmb \theta}_1, \ldots, {\pmb \theta}_M)$
such that
$\|{\pmb \Theta}\|_{P, \infty}\le \|f\|_{{\mathcal B}}
$
and
\begin{equation}\label{approximation.cor.eq1}
  \Big | \frac{1}{M} \sum_{m=1}^M \phi({\bf x}_0, {\pmb \theta}_m)-f({\bf x}_0)\Big| \le \frac{1+\epsilon}{\sqrt{M}} {\norm{f}_{\mathcal B}}.
\end{equation}
\end{corollary}

We remark that the shallow GCNN chosen in Corollary \ref{approximation.cor} may depend on  ${\bf x}_0\in \Omega$.

\medskip

For any $\epsilon>0$, we say that
the family of balls $B({\bf x}_i, \epsilon)$ with radius $\epsilon$ and center ${\bf x}_i\in \Omega, 1\le i\le I$,
is a  {\em $\epsilon$-covering} of the domain  $\Omega$ if
\begin{equation}\label{covering.def}
\Omega\subset \cup_{i=1}^I B({\bf x}_i, \epsilon),
\end{equation}
 and define the {\em $\epsilon$-covering number} $N^{\rm ext}_\epsilon$ by the minimal
number of balls in a $\epsilon$-covering of the domain $\Omega$.
Using the covering of the domain $\Omega$ with balls of small radius and the Lipschitz property for functions in the Barron space, we
next  consider the uniform approximation of  shallow GCNNs to  a function in the Barron space on the whole domain $\Omega$.

\begin{theorem}  \label{approximation.thm2}
Let $\epsilon\in (0, 1/2)$.
Assume that  the  ReLU function $\sigma$ satisfies \eqref{connorm0.def} and \eqref{lipschitz.eq2},    the convolution norm satisfies
\eqref{connorm.def} and
\eqref{connorm.def1}, and the integer $M\ge 1$ is chosen to
 satisfy
\begin{equation}\label{approximation.thm2.eq0}
2 N_{\epsilon}^{\rm ext} e^{-M \epsilon^2/2}<1,
\end{equation}
where $N_\epsilon^{\rm ext}$ is  the $\epsilon$-covering number of the domain $\Omega$.
Then for  any function $f$ in the Barron space ${\mathcal B}$
there exists a  shallow GCNN  with parameter  ${\pmb \Theta}=({\pmb \theta}_1, \ldots, {\pmb \theta}_M)$
such that
\begin{equation}\label{approximation.thm2.eq1}
\|{\pmb \Theta}\|_{P, \infty}\le \|f\|_{{\mathcal B}}
\end{equation}
and
\begin{equation}\label{approximation.thm2.eq2}
    \sup_{{\bf x}\in \Omega} \Big|\frac{1}{M} \sum_{m=1}^M \phi({\bf x}, {\pmb \theta}_m)-f({\bf x})\Big| \le (2D_1 \|\sigma\|_{\rm Lip}+1)\epsilon  \|f\|_{\mathcal B},
\end{equation}
 where $\|\sigma\|_{{\rm Lip}}$ and $D_1$
are the constants in  \eqref{lipschitz.eq2} and \eqref{connorm.def1} respectively.
\end{theorem}

The detailed proof of Theorem \ref{approximation.thm2} is given in Section \ref{approximation.thm2.pfsection}. 

For  the case that the unit ball $\Omega_{\rm sp}$  in \eqref{domain.def1}
is used as the domain $\Omega$ and the standard $p$-norm $\|\cdot\|_p$ selected as the norm $\|\cdot\|$,
 the $\epsilon$-covering number $N_\epsilon^{\rm ext}$ is bounded  above by ${{N}\choose{s}} (3/\epsilon)^s$ \cite{Vershynin2009}.
 This implies that the requirement
\eqref{approximation.thm2.eq0}  is met if
\begin{equation}\label{dimensioncurse.eq2}
M\ge \frac{2\ln 2}{\epsilon^2}+  \frac{2s}{\epsilon^2} \ln \Big(\frac{e N}{s\epsilon}\Big).\end{equation}
Hence  
 shallow GCNNs  with parameter size  $O(s N \epsilon^{-2} \ln (Ns^{-1}\epsilon^{-1}))$
  could be chosen to  approximate a function $f$ of sparse signals in the graph Barron space uniformly on the domain $\Omega_{\rm sp}$ with accuracy $\epsilon \|f\|_{\mathcal B}$.

For  the case that the unit ball $\Omega_{\rm ba}$ in \eqref{domain.def0}
is used as the domain $\Omega$, we have a better estimate on  the $\epsilon$-covering number $(1/\epsilon)^N\le N_\epsilon^{\rm ext}\le (3/\epsilon)^N$ and hence the requirement
\eqref{approximation.thm2.eq0} in Theorem \ref{approximation.thm2} is met  if
\begin{equation}\label{dimensioncurse.eq1}
M\ge \frac{2\ln 2}{\epsilon^2}+  \frac{2N}{\epsilon^2} \ln \Big(\frac{3}{\epsilon}\Big).\end{equation}
Therefore in the above setting on the domain and the norm,
 we conclude from Theorem \ref{approximation.thm2}  that for any $\epsilon\in (0,1/2)$,
   shallow GCNNs  with parameter size  $O(N^2 \epsilon^{-2} \ln (1/\epsilon))$
  could be selected to  approximate a function $f$ in the Barron space uniformly on the whole domain $\Omega_{\rm ba}$ with accuracy $\epsilon \|f\|_{\mathcal B}$,
  see \cite[Proposition 1]{Bach2017} for the parameter size of neural networks required for uniform approximation in the standard neuron network setting.

\subsection{Inverse and universal approximation theorems}
\label{inverseapproximation.subsection}

For $Q>0$, denote the set of functions in the Barron space with their Barron norms bounded by $Q$ by
\begin{equation} \label{FQ.def}
    {\mathcal F}_Q= \cb{ f\in \B\,:\, \|f\|_\B \le Q }.
\end{equation}
By \eqref{NQBarron.eq}, we have
\begin{equation}{\mathcal C}_{Q,p}\subset {\mathcal F}_Q.\end{equation}
Moreover, as a conclusion of  Theorem \ref{approximation.thm2}, the set ${\mathcal C}_{Q,p}$ has the following density property:
\begin{equation}\label{NQBarron.eq2+}
\inf_{g\in {\mathcal C}_{Q, p}}\sup_{{\bf x}\in \Omega} |g({\bf x})-f({\bf x})|=0
\end{equation}
hold for all $f\in {\mathcal F}_Q$  and $1\le p\le \infty$.
In the following theorem, we show that the converse  holds too.

\begin{theorem}\label{inverseapproximaton.thm}
Let $Q>0$ and $1\le p\le \infty$. If $f_n\in {\mathcal C}_{Q, p}, n\ge 1$, converges pointwise, i.e.,
\begin{equation} \lim_{n\to \infty} f_n({\bf x})= f({\bf x}), \ {\bf x}\in \Omega,
\end{equation}
for some  function $f$ on the domain $\Omega$.  Then $f\in {\mathcal B}$ and $\|f\|_{\mathcal B}\le Q$.
\end{theorem}

The detailed proof of the above inverse approximation theorem is given Section \ref{inverseapproximation.thm.pfsection}. A similar result is established in \cite[Theorem 2]{EMaWu2022} for  the classical neuron network setting.

\medskip

Let $C(\Omega)$ be the Banach space of continuous functions on the domain $\Omega$ with the norm
defined by
$\|f\|_\infty=\sup_{{\bf x}\in \Omega} |f({\bf x})|, \ f\in C(\Omega)$.
In the following theorem, we establish the universal approximation theorem
for the   shallow GCNN
when all graph signals are used  for the convolution operation of GCNNs.

\begin{theorem} \label{density.thm}  Let ${\bf U}$ be the orthogonal matrix in \eqref{diagonalization} to diagonalize the graph shifts ${\bf S}_1, \ldots, {\bf S}_K$
simultaneously.
If ${\mathcal W}={\mathbb R}^N$ and  there exists $1\le n_0\le N$ such that all entries in the $n_0$-th row of  the orthogonal matrix ${\bf U}$ are nonzero, then
for any $1\le p\le \infty$,
$\cup_{Q\ge 0} {\mathcal  C}_{Q, p}$ is  dense in $C({\Omega})$.
\end{theorem}

By \eqref{fM.eq1}, Corollary \ref{lipschitz.cor} and Theorem \ref{density.thm}, we have the following  density result. 

\begin{corollary} \label{univeral.cor}  Let ${\bf U}$ and ${\mathcal W}$ be as in Theorem \ref{density.thm}.
Then the Barron space ${\mathcal B}$ is  dense in $C({\Omega})$.
\end{corollary}

The proof of Theorem \ref{density.thm} is based on the following classical universal approximation theorem \cite{Pinkus1999}, see
Section \ref{density.thm.pfsection} for the detailed argument.

\begin{lemma}\label{universal.lem}
Let $f$ be a continuous function on the domain $\Omega$. Then for any $\epsilon>0$, there exist
 $u_m\in {\mathbb R}, {\bf v}_m\in {\mathbb R}^N$ and $ w_m\in {\mathbb R}, 1\le m\le M$
such that
\begin{equation}
\Big\|f({\bf x})-\sum_{m=1}^M u_m \sigma({\bf v}_m^T {\bf x}+w_m)\Big\|_\infty\le \epsilon.
\end{equation}
\end{lemma}

We remark that  in Theorem \ref{density.thm}, the assumption on the nonzero entries for the orthogonal matrix ${\bf U}$ can not be removed.
For instance, for an edgeless graph,  all graph shifts are diagonal matrices and hence we may select the unit matrix ${\bf I}$ as the  orthogonal matrix ${\bf U}$ to diagonal graph shifts. Thus the output $g$ of any shallow GCNN is separable, i.e.,  there exist functions $f_n, 1\le n\le N$, on the real line such that
$$g({\bf x})=\sum_{n=1}^N f_n(x_n),\  {\bf x}=[x_1, \ldots, x_N]\in \Omega.$$
Therefore for the edgeless graph setting, $\cup_{Q\ge 0} {\mathcal N}_{Q, p}, 1\le p\le \infty$, are not dense in $C({\Omega})$.

\section{Rademacher complexity}
\label{complexity.section}

 Rademacher complexity
 measures richness of a function class
 and it has been used to derive data-dependent upper-bounds on  learnability
 \cite{Bartlett2002, Yin2019}.
In this section, we consider the   Rademacher complexity of the family ${\mathcal F}_Q$ of functions on the domain $\Omega$,
\begin{equation}\label{Radcomplexity.def}
{\rm Rad}_S ({\mathcal F}_Q)={\mathbb E} \Big( \sup_{f\in {\mathcal F}_Q} \frac{1}{S}\sum_{i=1}^S \xi_i f({\bf x}_i)\Big),
        \end{equation}
        where
 ${\bf x}_i\in \Omega, 1\le i\le S$, are samples of
        ${\bf x}_i\in \Omega, 1\le i\le S$ in the domain $\Omega$,  $\xi_i\in \{-1, 1\}, 1\le i\le S$ are i.i.d. Rademacher random variables
        with $P(\xi_i= 1)=P(\xi_i=-1)=1/2$, and  ${\mathcal F}_Q, Q\ge 0$, contains all functions on the domain $\Omega$ with their Barron norms bounded by $Q$, see \eqref{FQ.def}.
   In the following theorem, we show that  the  Rademacher complexity ${\rm Rad}_S ({\mathcal F}_Q)$ may depend on
the negative square root of the sample size $S$ and the square root of the  logarithmic of the graph order $N$.

\begin{theorem}\label{Rademacher.thm}
Let ${\mathcal B}$ be the Barron space in \eqref{Barron.def00} with the norm  $\|\cdot\|$ in \eqref{connorm0.def}
replaced by the standard $\ell^\infty$-norm $\|\cdot\|_\infty$,  and
the convolution norm  $\|\cdot\|_{\rm co}$ satisfying the additional assumption \eqref{connorm.def2}.
For any ${\bf x}_i\in \Omega, 1\le i\le S$, define the
  Rademacher complexity  ${\rm Rad}_S({\mathcal F}_Q)$ of the family ${\mathcal F}_Q, Q\ge 0$,
 as in \eqref{Radcomplexity.def}.
Then
\begin{equation} \label{Rademacher.thm.eq1}
{\rm Rad}_S({\mathcal F}_Q)
  \le  2  Q \big(D_0D_2\sqrt{2\ln (2N)}+\sqrt{2\ln 2}\big)  S^{-1/2},
\end{equation}
where  $D_0$ and $D_2$ are the constants in \eqref{domain.def-1} and \eqref{connorm.def2} respectively.
\end{theorem}

Similar result in the standard neuron network setting can be found in  \cite{Bach2017, EMaWu2022}.  We follow the argument used in \cite[Theorem 3]{EMaWu2022}
in the proof of Theorem  \ref{Rademacher.thm} with details presented in
Section \ref{Rademacher.thm.pfsection}.

Following the standard procedure in \cite[Theorem 8]{Bartlett2002}, we  see that
functions in the Barron space can be learnt efficiently.

\begin{theorem}\label{error.thm}  Let  $Q\ge 0$,  $\mu$ be a probability measure on $\Omega$, and ${\bf x}_i\in \Omega, 1\le i\le S$, be i.i.d. random variables with probability measure $\mu$.
Set ${\bf X}=({\bf x}_1, \ldots, {\bf x}_S)$ and define
\begin{equation}\label{error.thm.eq1}\Phi({\bf X})=
\sup_{f\in {\mathcal F}_Q}\Big| \int_{\Omega} f({\bf x}) d\mu({\bf x})- \frac{1}{S} \sum_{i=1}^S f({\bf x}_i)\Big|.\end{equation}
Then for any  $\delta\in (0, 1/2)$, the error estimate
\begin{equation}\label{error.thm.eq2}
\Phi({\bf X})
\le   \big(4D_0D_2\sqrt{\ln (2N)}+4\sqrt{\ln 2}+ \sqrt {\ln (1/\delta)}\big)\sqrt{2} Q   S^{-1/2}
    \end{equation}
hold with probability at least $1-\delta$.
\end{theorem}

For the completeness of this paper, we include a brief proof in Section
\ref{error.thm.pfsection}.

\section{Numerical Simulations}
\label{numericaldemon.section}

In this section, we  consider  both synthetic and real data on the underlying undirected  graph  ${\mathcal G}=(V, E)$
of the data set of hourly temperature
collected at $32$ weather stations  in the region of Brest (France) \cite{perraudin17, segarra2017, Cong2023}.
 The temperature data set  
 is of size $32\times24\times 31$,
 and the weather station graph    is constructed by the 5 nearest neighboring stations in physical distances.
 In this section, we use stochastic gradient descent with Nesterov momentum to train shallow GCNNs on the weather station graph ${\mathcal G}$ and demonstrate
the approximation performance of shallow GCNNs presented in  Theorems \ref{approximation.thm} and \ref{approximation.thm2}.
 All experiments and computations  are performed using the PyTorch deep learning framework.

 Denote the symmetric normalized  Laplacian on the weather station graph ${\mathcal G}$ by ${\mathbf L}^{\rm sym}={\bf I}- {\bf D}^{-1/2} {\bf A} {\bf D}^{-1/2}$,
 where ${\bf A}$ and ${\mathbf D}$ are the adjacency and degree matrix of the graph ${\mathcal G}$ respectively. In our simulations, we set $N=32$,
 let the fundamental domain $\Omega$ of the GCNN contain all graph signals  ${\bf x}=(x(i))_{i\in V}$ with entries contained in $[-1, 1]$, i.e., $-1\le x(i)\le 1$ for all $i\in V$, and
we use
 $${\mathcal W}_L=\Big\{\sum_{l=0}^L b(l) ({\mathbf L}^{\rm sym})^l, \ b(0),  \ldots, b(L) \in {\mathbb R}\Big\}$$
 in \eqref{convolutionspace.def} as the convolution space.

 Given input  graph signals ${\bf x}_i\in {\mathbb R}^{N}, 1\le i\le S$, and output values $y_i=f({\bf x}_i), 1\le i\le S$ of  a function $f$ on the domain $\Omega$, we use
 stochastic gradient descent with Nesterov momentum $\mu=0.9$, SGDM for abbreviation,  as the optimization strategy
 to learn the parameters
 ${\pmb \Theta}$ of the desired GCNN to approximate the function $f$, see Algorithm \ref{sgdm.algorithm}. The loss function in the SGDM
 is the conventional relative mean square error (RMSE for abbreviation),
 	\begin{equation}\label{mseloss.def}
	F({\pmb \Theta}) = \frac{\sum_{i=1}^S  (y_i - f_M({\bf x}_i, {\pmb  \Theta}))^2}{\|{\bf y}\|_2^2 },
	\end{equation}
where ${\pmb \Theta}=({\pmb \theta}_1, \ldots, {\pmb \theta}_M)$ with
${\pmb \theta}_m=({\bf a}_m, {\bf b}_m, {\bf c}_m)\in {\mathbb R}^N\times {\mathbb R}^{L+1} \times {\mathbb R}^N$ and ${\bf b}_m=[b_m(0), \ldots, b_m(L)]^T\in {\mathbb R}^{L+1}, 1\le m\le M$, $\|{\bf y}\|_2=(\sum_{i=1}^S y_i^2)^{1/2}$, and
\begin{equation}\label{mseloss.def2}
    f_M ({\bf x}_i, {\pmb \Theta}) = \frac{1}{M} \sum_{m=1}^M {\bf a}_{m}^T \sigma
     \Big(\sum_{l=0}^L b_m(l) ({\mathbf L}^{\rm sym})^l{\bf x}_i+ {\bf c}_{m}\Big), 1\le i\le S.
\end{equation}
 For the case that  ${\bf x}_i, 1\le i\le S$, are randomly and independently selected  with uniform distribution on $\Omega:=[-1, 1]^N$,
 for  large sampling size $S$ we may show that  the  RMSE $F({\pmb \Theta})$ is  about
 the relative approximation error of the GCNN in square norm,
 $$ F({\pmb \Theta})\approx \frac{\int_{\Omega} |f({\bf x})-f_M({\bf x})|^2 d{\bf x}} {\int_{\Omega} |f({\bf x})|^2 d{\bf x}}.$$

As the ReLU function $\sigma$ is not differentiable, we define its  approximate derivative
$\sigma_{\rm app}^\prime$ by
$$\sigma_{\rm app}^\prime(t)=
\frac{\sigma(t+\epsilon)-\sigma(t-\epsilon)}{2\epsilon}=\left\{\begin{array}{ll}
1 & {\rm if} \  t\in [\epsilon, +\infty)\\
\frac{t+\epsilon}{2\epsilon}& {\rm if} \  t\in (-\epsilon, \epsilon)\\
0 & {\rm if} \ t\in (-\infty,  -\epsilon],
\end{array}\right.
$$
which is also
the derivative of the regularization
$$\sigma_{\epsilon}(t)=  \left\{\begin{array}{ll}
t & {\rm if} \  t\in [\epsilon, +\infty)\\
{(t+\epsilon)^2}/{(4\epsilon)}& {\rm if} \  t\in (-\epsilon, \epsilon)\\
0 & {\rm if} \ t\in (-\infty,  -\epsilon]
\end{array}\right.
$$
of the ReLU function $\sigma$,
where $\epsilon=10^{-5}$ in our simulations.
Set
$${\bf b}_m({\mathbf L}^{\rm sym})=\sum_{l=0}^L b_m(l) ({\mathbf L}^{\rm sym})^l, \ 1\le m\le M$$
and
$${\bf z}_i({\pmb \Theta})=y_i - f_M({\bf x}_i, {\pmb  \Theta}),\  1\le i\le S.$$
In the SGDM,   we use the approximate gradient $\nabla F_{\rm app}$ of the loss function $F$:
\begin{equation}
\frac{\partial F_{\rm app}({\pmb \Theta})}{\partial {\bf a}_m}=-\frac{2}{M\|{\bf y}\|_2^2 } \sum_{i=1}^S { z}_i({\pmb \Theta})
\sigma
     \big({\bf b}_m({\mathbf L}^{\rm sym}){\bf x}_i+ {\bf c}_{m}\big),
\end{equation}
\begin{equation}
\frac{\partial F_{\rm app}({\pmb \Theta})}{\partial {\bf c}_m}=-\frac{2}{M\|{\bf y}\|_2^2 } \sum_{i=1}^S { z}_i({\pmb \Theta})
{\rm diag} \big( \sigma_{\rm app}'     \big({\bf b}_m({\mathbf L}^{\rm sym}){\bf x}_i+ {\bf c}_{m}\big)\big) {\bf a}_m,
\end{equation}
and
\begin{equation}
\frac{\partial F_{\rm app}({\pmb \Theta})}{\partial {{\bf b}_m(l)}}=-\frac{2}{M\|{\bf y}\|_2^2 } \sum_{i=1}^S { z}_i({\pmb \Theta}) {\bf a}_m^T
{\rm diag} \big( \sigma_{\rm app}'     \big({\bf b}_m({\mathbf L}^{\rm sym}){\bf x}_i+ {\bf c}_{m}\big)\big)
 ({\mathbf L}^{\rm sym})^l{\bf x}_i,
\end{equation}
where  $1\le m\le M$ and $0\le l\le L$.

\begin{algorithm}
[t]
\caption{Stochastic Gradient Descent  Algorithm with Nesterov Momentum to learn GCNNs}
\label{sgdm.algorithm}
\begin{algorithmic}

\STATE {\em Inputs}: Order $L$ of polynomial filter, number $M$ of neurons, order $N$ of the underlying graph ${\mathcal G}$, number $S$ of samples,
 Nesterov momentum $\mu=0.9$, learning rate $\gamma=0.003$, number of iteration ${\rm Iter}$,
the symmetrically normalized Laplacian ${\bf L}^{\rm sym}$, input signals ${\bf x}_i\in {\mathbb R}^N, 1\le i\le S$, and   outputs $y_i=f({\bf x}_i), 1\le i\le S$
of the function $f$ to be learnt.

\STATE {\em Initial}:   ${\pmb \Theta} \leftarrow {\bf 0}$ and ${\bf Temp} \leftarrow \nabla F_{\rm app}({\pmb \Theta})$.

\STATE {\em Iteration}:
\vskip-8mm
\begin{align*}
    &\textbf{for}\: n=1\\
            &\hspace{10mm} {\bf Grad}\leftarrow     \nabla F_{\rm app}({\pmb \Theta})                   \\
            &\hspace{10mm} {\bf Temp}  \leftarrow \mu * {\bf Temp} +{\bf Grad}         \\
            &\hspace{10mm} {\pmb \Theta}  \leftarrow {\pmb \Theta}-\gamma*{\bf Temp}                                         \\
            &\hspace{10mm} n \leftarrow n+1                   \\
    &\textbf{stop}\: {\rm if}\  n > {\rm Iter}                                 \\
            \end{align*}
\vskip-10mm
\STATE {\em Output}: $ {\pmb \Theta}_{\rm Iter}$.
\end{algorithmic}
\end{algorithm}

\medskip

In our first simulation, we consider the quadratic function
\begin{equation} \label{randomquadraticfunction.def} f(x)= \|{\bf B}{\bf x}\|_2^2,\  {\bf x}\in \Omega,\end{equation}
where ${\bf B}=(b(i,j))_{i,j\in V}$
has zero entries except that $b(i, i), i\in V$ and $b(i,j), (i,j)\in E$ are taken randomly and  independently with uniform distribution on $[-1, 1]$.
To learn the above function $f$ from the SGDM, we assume that
 the given input graph signals ${\bf x}_i, 1\le i\le S$, are randomly and independently selected with uniform distribution on $[-1, 1]^{N}$, and the output values
 $y_i=f(x_i)=\|{\bf B}{\bf x}_i\|_2^2, 1\le i\le S$.
Shown in  Figure  \ref{square1.fig} is the performance of the SGDM to learn the function $f$ from its sampling data $({\bf x}_i, f({\bf x}_i)), 1\le i\le S$.
 We observe from Figure \ref{square1.fig}
  that  the SGDM
  converges and has better performance when the sampling size $S$ increases.
  This demonstrates the theoretical result in Theorem \ref{error.thm} on  higher learnability of functions
   from their random samples of larger size.
   \begin{figure}[t] 
\centering
\includegraphics[width=68mm, height=48mm]{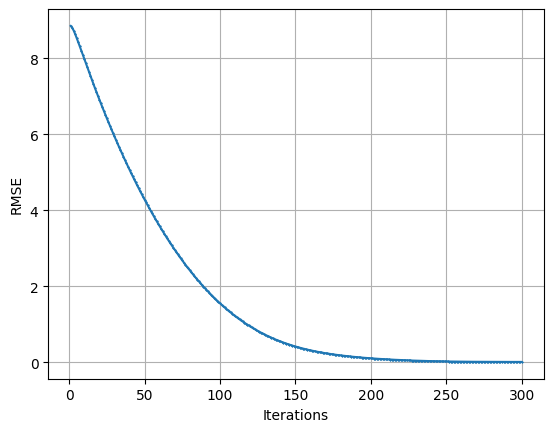}
\includegraphics[width=68mm, height=48mm]{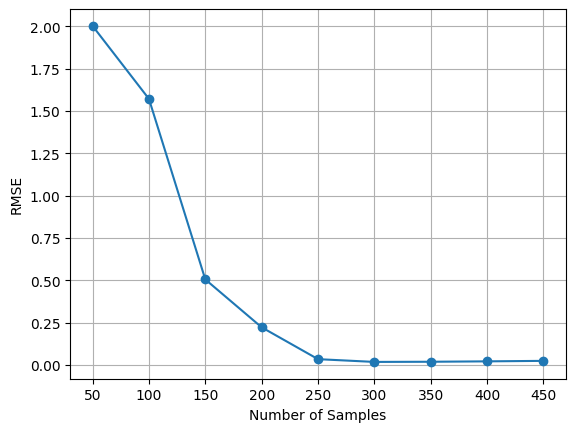}
\caption{Plotted on the  left is the average  RMSE $F({\pmb \Theta}_n),
 1\le n\le 300$, over 100 trials, where $M=4, L=5$ and $S=100$
and ${\pmb \Theta_n}, 1\le n\le 300$, are the parameters  in  the
  $n$-th iteration of the SGDM.
   The average energy  $S^{-1} \|{\bf y}\|_2^2$  of the output data over 100 trial is $21.8736$, 
   while the uniform bound $\|{\bf y}\|_\infty=\sup_{1\le i\le S} |y_i|$ over 100 trial is $6.4467$. 
   Presented on the right is the
the average  RMSE $F({\pmb \Theta}_{\rm Iter})$ over 100 trials with respect to different sampling size $S$,
where  $M=4$, $L=5$ and ${\rm Iter}=100$.
}\label{square1.fig}
\end{figure}

Define the relative uniform approximation error  (RUAE) of the GCNN with parameter ${\pmb \Theta}$
  by
$$U ({\pmb \Theta}):= \frac{\sup_{1\le i\le S}|f({\bf x}_i)-f_M({\bf x}_i, {\pmb \Theta})|}
{\sup_{1\le i\le S} |f({\bf x}_i)|+10^{-6}},$$
where $f$ is the original function and $f_M({\bf x}, {\pmb \Theta})$ is the out of the GCNN given  in \eqref{mseloss.def2}.
   Presented in  Figure  \ref{square.fig} is  how the RMSE and RUAE vary with the number of neurons per vertex.
   This demonstrates the theoretical result in Theorems \ref{approximation.thm} and \ref{approximation.thm2}
 on the approximation property of GCNNs.
We observe from Figure \ref{square.fig} that  increasing the number of neurons at each vertex generally improves the accuracy of the GCNN, as measured by both RMSE and RUAE, as long as the number of iterations in SGDM is not too high.
 However, when the number of iterations is high (then RMSE and RUAE are low), adding more neurons does not help and may even hurt the performance of the GCNN.
We hypothesize that this is because the GCNN becomes overfitted to the training data and loses its ability to generalize to new data.

\begin{figure}[t] 
\centering
\includegraphics[width=65mm, height=48mm]{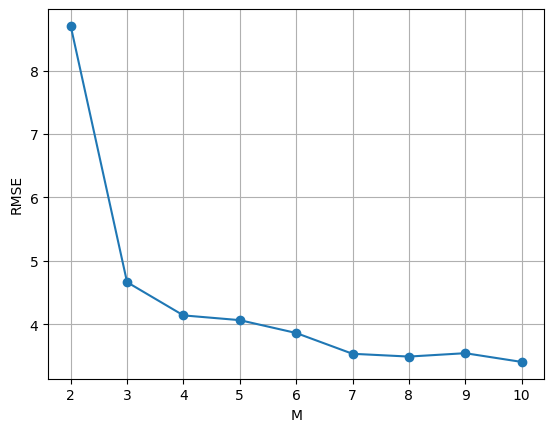}
\includegraphics[width=65mm, height=48mm]{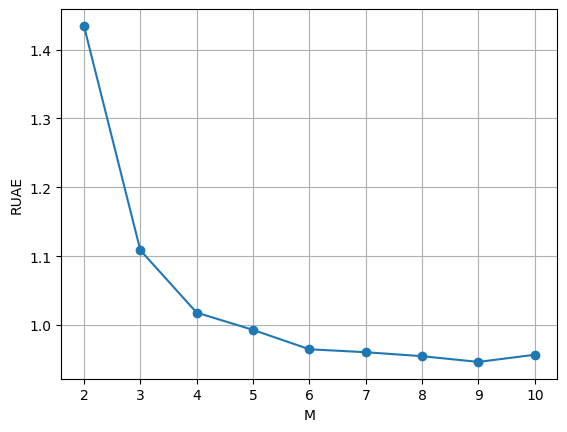}\\
\includegraphics[width=65mm, height=48mm]{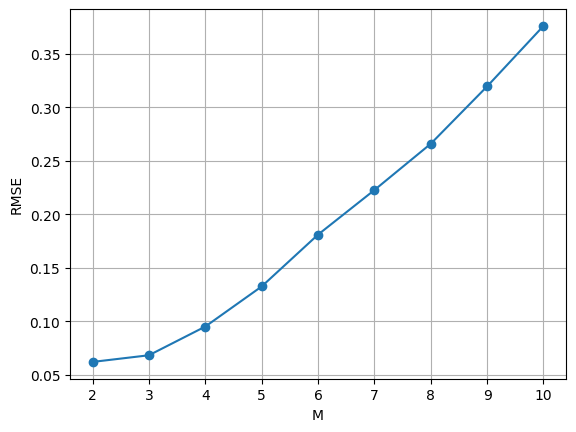}
\includegraphics[width=65mm, height=48mm]{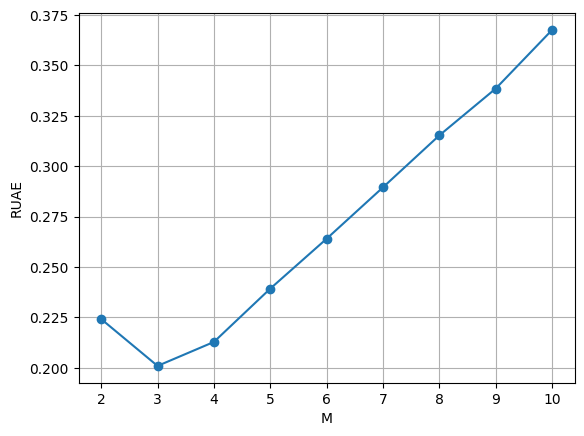}
\caption{Plotted  are the
 average RMSE $F({\pmb \Theta}_{\rm Iter})$ (left) and RUAE $U({\pmb \Theta}_{\rm Iter})$ (right) over 100 trials with respect to different number  $M$ of neurons per vertex,
where  $S=100, L=5$, and ${\rm Iter}= 50$ (top left and right) and $200$ (bottom left and right) respectively. }
\label{square.fig}
\end{figure}

\medskip

In the second simulation, we consider  the real data set
  of hourly temperature measured in Celsius
collected at $32$ weather stations  in the region of Brest (France) in
 January 2014.
 Denote the regional temperature  at $t_i$-th hour of  $d$-th day by
   ${\bf x}^{\rm org}_d(t_i), 0\le i\le 23, 1\le d\le 31$.
Before we apply GCNNs to learn functions,   we  pre-process the temperature data set
by eliminating the average temperature and rescaling the range to $[-1, 1]$,
$${\bf x}_d(t_i)= B^{-1}\big({\bf x}^{\rm org}_d(t_i)-{\bf x}^{\rm org}_{\rm ave}\big),$$
where  ${\bf x}^{\rm org}_{\rm ave}= (23\times 31)^{-1}
 \sum_{i=0}^{23} \sum_{d=1}^{31} {\bf x}_d^{\rm org}(t_i)$  the average temperature in the region of Brest (France)
 for January 2014,  and  $B$ is chosen so that ${\bf x}_d(t_i)\in [-1, 1]$ for all $1\le d\le 31$ and $0\le i\le 23$. In particular,
 we take $B=10.35$ 
 in our simulation.
 In the second simulation, we want to learn GCNNs to approximate the squared variance function $f_{\rm sv}$  of next day,
 \begin{equation}\label{fsv.def00}
  f_{\rm sv}({\bf x}_d(t_i))=  \|{\bf x}_{d+1}(t_i)\|_2^2- ({\bar {\bf x}}_{d+1}(t_i))^2, \    1\le d\le 30, 0\le i\le 23,
   \end{equation}
   where ${\bar {\bf x}}_{d+1}(t_i)$ is the average pre-processed temperature data of the whole Brest region at $t_i$-th hour of  $(d+1)$-th day.

 Learning GCNNs from real-world data is a challenging 
 task. In the second simulation, we try to learn GCNN from about 20\% of the weather data set, particularly,
  ${\bf x}_d(t_i)$ and $f_{\rm sv}, 0\le i\le 23,  d\in \{1, 6, 11, 16, 21, 26\}$,
 to learn the squared variance function  $f_{\rm sv}$ of next day.
Shown in Figure \ref{weatherdata.fig} is the approximation property of the output of the GCNN obtained from the SGDM, where
the relative mean square error (RMSE) and relative uniform approximation error  (RUAE) on the {\bf whole} weather data set are defined by
$${\rm WMSE}=\frac{ \sum_{i=0}^{23} \sum_{d=1}^{30} |f_{\rm sv}({\bf x}_d(t_i))-f_M({\bf x}_d(t_i), {\pmb \Theta}_{\rm Iter})|^2}
{\sum_{i=0}^{23} \sum_{d=1}^{30} |f_{\rm sv}({\bf x}_d(t_i))|^2 }
$$
and on the right is the uniform error
$${\rm WUAE}=\frac{\sup_{1\le d\le 30, 0\le i\le 23} |f_{\rm sv}({\bf x}_d(t_i))-f_M({\bf x}_d(t_i), {\pmb \Theta}_{\rm Iter})|}
{\sup_{1\le d\le 30, 0\le i\le 23} |f_{\rm sv}({\bf x}_d(t_i)|},$$
where ${\pmb \Theta}_{\rm Iter}$ is the output parameter of the SGDM.
Comparing with the approximation of the quadratic function $f$ in \eqref{randomquadraticfunction.def} with the squared variance function $f_{\rm sv}$ in \eqref{fsv.def00} by GCNNs, the number of neurons at each vertex has a positive impact on the accuracy of the GCNN, when the number of iterations in SGDM is not too high. However, when the number of iterations is high, adding more neurons does not improve and may even degrade the performance of the GCNN.

From the approximation property presented by
 Figures \ref{square.fig} and \ref{weatherdata.fig}, we observe that
there is a trade-off between the number $M$ of neurons and the number ${\rm Iter}$ of iterations in the SGDM that needs to be carefully balanced to achieve the optimal performance of the GCNNs.

 \begin{figure}[t] 
\centering
\includegraphics[width=68mm, height=48mm]{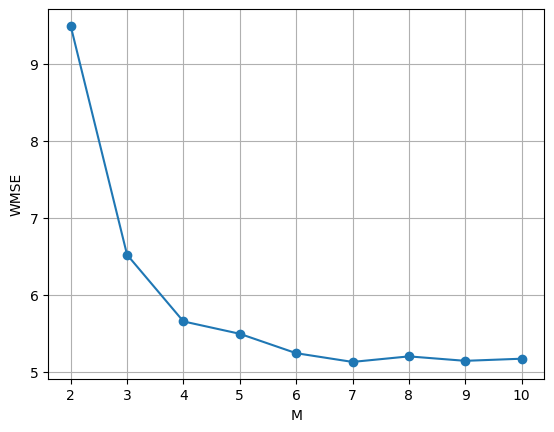}
\includegraphics[width=68mm, height=48mm]{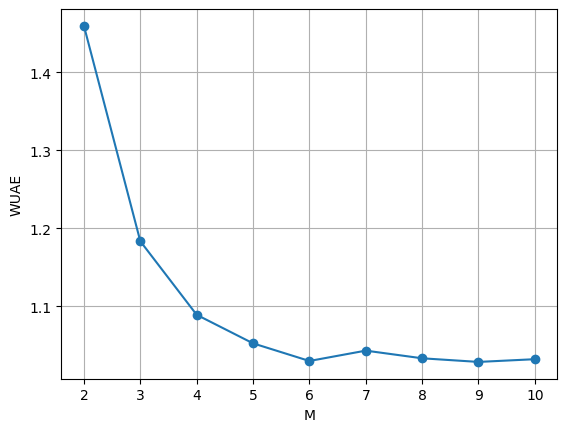}\\
\includegraphics[width=68mm, height=48mm]{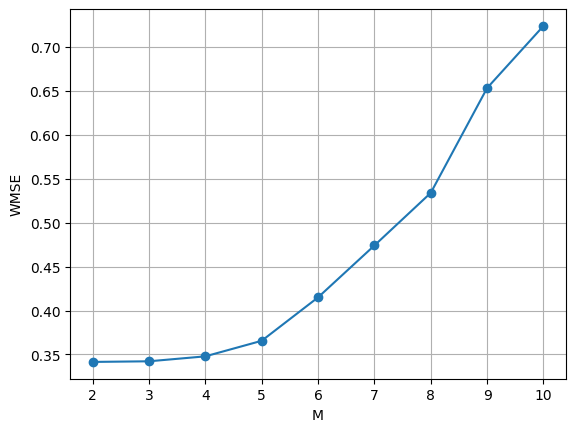}
\includegraphics[width=68mm, height=48mm]{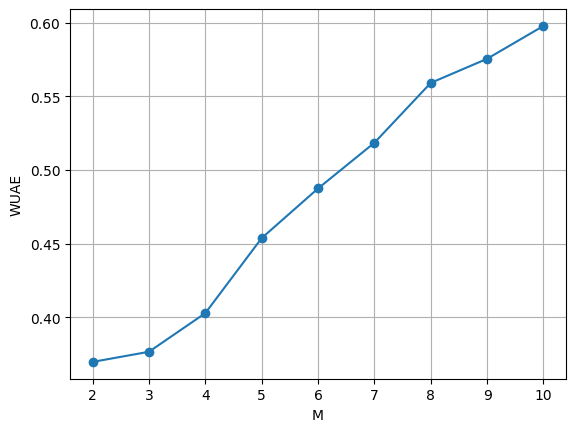}
\caption{Plotted  are the
 average WMSE  (left) and WUAE  (right) over 100 trials with respect to different number  $M$ of neurons per vertex,
where  $S=218=24\times 12, L=5$, and ${\rm Iter}= 30$ (top left and right) and $100$ (bottom left and right) respectively.
}
\label{weatherdata.fig}
\end{figure}

\section{Proofs}
\label{proofs.section}

In this section, we collect the proofs of Theorems  \ref{BanachBarronspace.thm}, \ref{rkhs.thm},  \ref{BarronRkhs.thm}, \ref{approximation.thm}, \ref{approximation.thm2}, \ref{inverseapproximaton.thm}, \ref{density.thm}, \ref{Rademacher.thm}  and \ref{error.thm}.

\subsection{Proof of Theorem \ref{BanachBarronspace.thm}}
\label{BanachBarronspace.thm.pfsection}

To prove Theorem \ref{BanachBarronspace.thm}, we first show that $\|\cdot\|_{{\mathcal B}_1}$  in \eqref{barronnorm.def}
defines a norm on  the Barron space ${\mathcal B}_1$.

\begin{lemma}\label{quasinorm.lem} Let $\|\cdot\|_{{\mathcal B}_1}$ be as in \eqref{barronnorm.def}.
Then 
\begin{itemize}
\item[{(i)}]  $f=0$ if and only if $\|f\|_{{\mathcal B}_1}=0$.
\item[{(ii)}] $\|\alpha f\|_{{\mathcal B}_1}= |\alpha| \|f\|_{{\mathcal B}_1}$ for all $f\in {\mathcal B}_1$ and $\alpha\in {\mathbb R}$.

\item[{(iii)}] $\|f+g\|_{{\mathcal B}_1}\le \|f\|_{{\mathcal B}_1}+ \|g\|_{{\mathcal B}_1}$ for all $f, g\in {\mathcal B}_1$.

\end{itemize}
\end{lemma}

\begin{proof} (i)\ \ Taking the Dirac measure  $\delta_0$ at the origin  as the probability measure
in \eqref{functionrepresentation.def} gives a representation for the zero function. This shows that $\|f\|_{{\mathcal B}_1}=0$ for the zero function $f=0$.  Conversely,
given  $f\in {\mathcal B}_1$ with $\|f\|_{{\mathcal B}_1}=0$, there exists a probability measure $\rho$ for any $\epsilon>0$ such that
\eqref{functionrepresentation.def} holds and
${\mathbb E}_\rho (\|{\bf a}\|_*(\|{\bf b}\|_{\rm co}+\|{\bf c}\|))\le \epsilon$.
Therefore for any ${\bf x}\in \Omega$, we have
\begin{eqnarray*}
|f({\bf x})| &\hskip-0.08in  \le &  \hskip-0.08in \int_{{\mathbb R}^N \times {\mathcal W}\times {\mathbb R}^N} \|{\bf a}\|_*(\|{\bf b}\|_{\rm co}+\|{\bf c}\|) \rho(d{\bf a}, d{\bf b}, d{\bf c})\nonumber\\
 & \hskip-0.08in= & \hskip-0.08in   {\mathbb E}_\rho (\|{\bf a}\|_*(\|{\bf b}\|_{\rm co}+\|{\bf c}\|))\le \epsilon.
\end{eqnarray*}
As $\epsilon>0$ is arbitrary chosen, we conclude that $f$ must be the zero function. This proves the conclusion (i).

\medskip
(ii)\ \ Clearly it suffices to show that
\begin{equation}\label{quasinorm.thm.pfeq1}
\|\alpha f\|_{{\mathcal B}_1}\le |\alpha| \|f\|_{{\mathcal B}_1}
\end{equation}
for all $0\ne \alpha\in {\mathbb R}$ and $f\in {\mathcal B}_1$.
Take arbitrary $\epsilon>0$ and let $\rho$ be a probability measure in ${\mathcal P}_f$ such that
\begin{equation}\label{quasinorm.thm.pfeq2}
{\mathbb E}_\rho (\|{\bf a}\|_*(\|{\bf b}\|_{\rm co}+\|{\bf c}\|))\le \|f\|_{{\mathcal B}_1}+\epsilon.\end{equation}
Define a new probability measure $\tilde \rho (A)= \rho(\tilde A)$
for any Borel set $A$, where $\tilde A=\{({\bf a}, {\bf b}, {\bf c})| ({\bf a}/\alpha, {\bf b}, {\bf c})\in A\}$.
Then one may verify that
\begin{eqnarray*} \alpha f({\bf x}) & \hskip-0.08in = & \hskip-0.08in \int_{({\bf a}, {\bf b}, {\bf c})\in {\mathbb R}^N\times {\mathcal W}\times {\mathbb R}^N}
(\alpha {\bf a})^T \sigma({\bf b}*{\bf x}+{\bf c}) \rho(d{\bf a}, d{\bf b}, d{\bf c}) \nonumber\\
& \hskip-0.08in = & \hskip-0.08in
\int_{(\tilde {\bf a}, {\bf b}, {\bf c})\in {\mathbb R}^N\times {\mathcal W}\times {\mathbb R}^N}
\tilde {\bf a}^T \sigma({\bf b}*{\bf x}+{\bf c}) \tilde \rho(d\tilde {\bf a}, d{\bf b}, d{\bf c})
\end{eqnarray*}
and
\begin{eqnarray*}
 {\mathbb E}_{\tilde \rho} \big(\|{\bf a}\|_*(\|{\bf b}\|_{\rm co}+\|{\bf c}\|)\big )
 & \hskip-0.08in = & \hskip-0.08in
\int_{({\bf a}, {\bf b}, {\bf c})\in {\mathbb R}^N\times {\mathcal W}\times {\mathbb R}^N}
 \|{\bf a}\|_*(\|{\bf b}\|_{\rm co}+\|{\bf c}\|) \tilde \rho(d{\bf a}, d{\bf b}, d{\bf c})\\
 & \hskip-0.08in = & \hskip-0.08in
\int 
 \|\alpha \tilde {\bf a}\|_*(\|{\bf b}\|_{\rm co}+\|{\bf c}\|)  \rho(d\tilde {\bf a}, d{\bf b}, d{\bf c})\nonumber\\
 & \hskip-0.08in = & \hskip-0.08in |\alpha| {\mathbb E}_{\tilde \rho} (\|\tilde {\bf a}\|_*(\|{\bf b}\|_{\rm co}+\|{\bf c}\|))
 \le |\alpha| \|f\|_{{\mathcal B}_1}+ |\alpha| \epsilon.
\end{eqnarray*}
Then the desired estimate \eqref{quasinorm.thm.pfeq1} follows from the above estimate and  the arbitrary selection of $\epsilon>0$.

\medskip

(iii) \ \  
By the second conclusion, it suffices to prove that
\begin{equation}\label{quasinorm.thm.pfeq3}
\|\alpha f_1+ (1-\alpha)f_2 \|_{{\mathcal B}_1}\le  \alpha \|f_1\|_{{\mathcal B}_1}+(1-\alpha) \|f_2\|_{{\mathcal B}_1}, \ f_1, f_2\in {\mathcal B}_1
\end{equation}
where $0\le \alpha\le 1$.
Take arbitrary $\epsilon>0$ and let $\rho_1\in {\mathcal P}_{f_1}$ and $\rho_2\in {\mathcal P}_{f_2}$ be
two probability measures so that
\begin{equation}\label{quasinorm.thm.pfeq4}
{\mathbb E}_{\rho_l} (\|{\bf a}\|_*(\|{\bf b}\|_{\rm co}+\|{\bf c}\|))\le \|f_l\|_{{\mathcal B}_1}+\epsilon, \ l=1, 2.
\end{equation}
Define $\rho=\alpha \rho_1+ (1-\alpha)\rho_2$ and set $f=\alpha f_1+ (1-\alpha) f_2$. Then one may verify that  $\rho$ is a probability measure in ${\mathcal P}_{f}$ and
\begin{equation*}
{\mathbb E}_{\rho} (\|{\bf a}\|_*(\|{\bf b}\|_{\rm co}+\|{\bf c}\|))
\le  \alpha\|f_1\|_{{\mathcal B}_1}+(1-\alpha)\|f_2\|_{{\mathcal B}_1}+ \epsilon.
\end{equation*}
This together with the arbitrary selection of $\epsilon>0$ proves \eqref{quasinorm.thm.pfeq3}.
\end{proof}

To prove Theorem \ref{BanachBarronspace.thm}, we next show that the probability measure $\rho$ in the representation \eqref{functionrepresentation.def}
of  any function $f$ in ${\mathcal B}_1$ could be selected to be supported on the dilated unit
sphere.

\begin{lemma}
\label{quasinorm2.lem}
Let ${\mathbb S}$ and ${\mathbb T}$ be as in \eqref{ST.def}.
Then for any $ f\in {\mathcal B}_1$, there exists a probability measure $\hat \rho$
supported on ${\mathbb S}\times {\mathbb T}$ such that
\begin{equation} \label{quasinorm2.eq1}
f({\bf x})=\|f\|_{{\mathcal B}_1} \int_{{\mathbb S}\times {\mathbb T}} {\bf a}^T
\sigma({\bf b}*{\bf x}+{\bf c}) \hat\rho(d{\bf a}, d{\bf b}, d{\bf c}), \  {\bf x}\in \Omega.\end{equation}
\end{lemma}

\begin{proof}  The conclusion is obvious for the zero function. Now we assume that $f\ne 0$.
 By \eqref{barronnorm.def}, there exist probability measures $\rho_n\in {\mathcal P}_f, n\ge 1$, on ${\mathbb R}^N\times {\mathcal W}\times {\mathbb R}^N$ such that
\begin{equation} \label{barronnorm2.prop.pfeq2}
\|f\|_{{\mathcal B}_1}\le A_n:={\mathbb E}_{\rho_n}\big(\|{\bf a}\|_* (\|{\bf b}\|_{\rm co}+\|{\bf c}\|)\big)\le \|f\|_{{\mathcal B}_1}+2^{-n}, \ n\ge 1.
\end{equation}
For $n\ge 1$, define  the measure $\hat\rho_n(E)$ of a Borel measurable subset  $E\subset {\mathbb S}\times {\mathbb T}$ by
\begin{equation*}
\hat \rho_n (E)=A_n^{-1} {\mathbb E}_{\rho_n} \big( \|{\bf a}\|_* (\|{\bf b}\|_{{\rm co}}+\|{\bf c}\|)\chi_{\hat E} ({\bf a}, {\bf b}, {\bf c})\big)
\end{equation*}
where
$${\hat E}=\Big \{({\bf a}, {\bf b}, {\bf c})\in {\mathbb R}^N\times {\mathcal W}\times {\mathbb R}^N:\
 \Big(\frac{  {\bf a}}{\|{\bf a}\|_*}, \frac{{\bf b}}{\|{\bf b}\|_{\rm co}+\|{\bf c}\|},
 \frac{{\bf c}}{\|{\bf b}\|_{\rm co}+\|{\bf c}\|}\Big)\in E\Big\},$$
and  $\chi_{\hat E}$ is the characteristic function on the set $\hat E$.
One may verify that $\hat \rho_n,  n\ge 1$, are probability measures on
${\mathbb S}\times {\mathbb T}$, and
\begin{eqnarray} \label{barronnorm2.prop.pfeq3}
f({\bf x}) & = & \int_{{\mathbb R}^N\times {\mathcal W}\times {\mathbb R}^N} {\bf a}^T \sigma({\bf b}*{\bf x}+{\bf c}) \rho_n ( d{\bf a}, d {\bf b}, d{\bf c})\nonumber\\
& = & \int_{{\mathbb R}^N\times {\mathcal W}\times {\mathbb R}^N}  \|{\bf a}\|_* (\|{\bf b}\|_{\rm co}+\|{\bf c}\|) \nonumber\\
& & \times  \Big(\frac{{\bf a}}{\|{\bf a}\|_*}\Big )^T
\sigma\Big(\frac{\bf b}{\|{\bf b}\|_{\rm co}+\|{\bf c}\|}*{\bf x}+\frac{{\bf c}}{\|{\bf b}\|_{\rm co}+\|{\bf c}\|}\Big)
 \rho_n ( d{\bf a}, d {\bf b}, d{\bf c})\nonumber\\
 & = & A_n \int_{{\mathbb S}\times {\mathbb T}}
 {\hat {\bf a}}^T
\sigma\big(\hat {\bf b}*{\bf x}+\hat {\bf c})
 \hat \rho_n ( d\hat{\bf a}, d \hat{\bf b}, d\hat{\bf c}).
\end{eqnarray}
Recall that  $\hat \rho_n, n\ge 1$, is a sequence of probability measures
on the compact set ${\mathbb S}\times {\mathbb T}$. Then by  Prokhorov theorem \cite{Prokhorov1956},   without loss of generality, we assume that $\hat \rho_n, n\ge 1$, converges weakly to a probability measure $\hat \rho$ on ${\mathbb S}\times {\mathbb T}$
\begin{equation} \label{barronnorm2.prop.pfeq4}
\lim_{n\to \infty} \hat \rho_n=\hat \rho \ \ {\rm weakly},
\end{equation}
otherwise replacing the sequence by a weakly convergent subsequence.
For any ${\bf x}\in\Omega$, the function ${\bf a}^T ({\bf b}*{\bf x}+{\bf c})$ is continuous with respect to $({\bf a}, {\bf b}, {\bf c})\in {\mathbb S}\times {\mathbb T}$ and is bounded by one. Therefore
the desired conclusion \eqref{quasinorm2.eq1}
follows from \eqref{barronnorm2.prop.pfeq2}, \eqref{barronnorm2.prop.pfeq3} and \eqref{barronnorm2.prop.pfeq4}.
\end{proof}

Now we are ready to prove Theorem \ref{BanachBarronspace.thm}.

\begin{proof}[Proof of Theorem \ref{BanachBarronspace.thm}]  First we prove that ${\mathcal B}_1$ is a Banach space.
By  Lemma \ref{quasinorm.lem}, it suffices to prove every Cauchy sequence  $f_n, n\ge 1$,
 in ${\mathcal B}_1$ converges to some function in ${\mathcal B}_1$.
 In particular,
 without loss of generality, we may assume that
 \begin{equation} \label{BanachBarronspace.thm.pfeq1}
 \|f_{n+1}-f_{n}\|_{{\mathcal B}_1}\le 2^{-n}, \ n\ge 1,
 \end{equation}
 other replacing it by one of its subsequences satisfying \eqref{BanachBarronspace.thm.pfeq1}.

 By  Lemma \ref{quasinorm2.lem}, there exist probability  measures $\hat \rho_n, n\ge 1$, on ${\mathbb S}\times {\mathbb T}$
 such that
\begin{eqnarray} \label{BanachBarronspace.thm.pfeq2}
f_{1}({\bf x})=\|f_{1}\|_{{\mathcal B}_1} \int_{{\mathbb S}\times {\mathbb T}} {\bf a}^T
\sigma({\bf b}*{\bf x}+{\bf c}) \hat\rho_1(d{\bf a}, d{\bf b}, d{\bf c})
\end{eqnarray}
and
\begin{equation} \label{BanachBarronspace.thm.pfeq3}
f_{n}({\bf x})- f_{n-1}({\bf x})=\|f_{n}-f_{n-1}\|_{{\mathcal B}_1} \int_{{\mathbb S}\times {\mathbb T}} {\bf a}^T
\sigma({\bf b}*{\bf x}+{\bf c}) \hat\rho_n(d{\bf a}, d{\bf b}, d{\bf c}), \  {\bf x}\in \Omega\end{equation}
for all $n\ge 2$.
 Define
\begin{equation}\label{BanachBarronspace.thm.pfeq4} f({\bf x})= A \int_{{\mathbb S}\times {\mathbb T}} {\bf a}^T
\sigma({\bf b}*{\bf x}+{\bf c}) \hat\rho(d{\bf a}, d{\bf b}, d{\bf c}), \  {\bf x}\in \Omega,
\end{equation}
where $A=\|f_{1}\|_{{\mathcal B}_1} +\sum_{n=2}^\infty  \|f_{n}-f_{n-1}\|_{{\mathcal B}_1}<\infty$ and
 the probability measure  $\hat \rho$  on ${\mathbb S}\times {\mathbb T}$  is given  by
\begin{equation}\label{BanachBarronspace.thm.pfeq5}\hat \rho= A^{-1}\Big (\|f_{1}\|_{{\mathcal B}_1} \hat \rho_1+ \sum_{n=2}^\infty  \|f_{n}-f_{n-1}\|_{{\mathcal B}_1} \hat \rho_n\Big).\end{equation}

Dilate the measure $\hat\rho$ on ${\mathbb S}\times {\mathbb T}$
to a probability measure on $(A{\mathbb S})\times {\mathbb T}$
and then extend  to a probability measure on ${\mathbb R}^N\times {\mathcal W}\times {\mathbb R}^N$ with support on
$(A{\mathbb S})\times {\mathbb T}$.  Denote the dilated extension measure by $\tilde \rho$.
By  \eqref{BanachBarronspace.thm.pfeq4} and \eqref{BanachBarronspace.thm.pfeq5}, the dilated extension measure $\tilde \rho$ is a probability measure
on ${\mathbb R}^N\times {\mathcal W}\times {\mathbb R}^N$ satisfying \eqref{functionrepresentation.def}, i.e.,  $\tilde \rho\in {\mathcal P}_f$.
Again  from  \eqref{BanachBarronspace.thm.pfeq4}  and \eqref{BanachBarronspace.thm.pfeq5} we obtain
\begin{equation*}\|f\|_{{\mathcal B}_1}\le {\mathbb E}_{\tilde \rho} (\|{\bf a}\|_* (\|{\bf b}\|_{\rm co}+\|{\bf c}\|))= A,
\end{equation*}
which  proves that $f\in {\mathcal B}_1$.

Dilate and extend the probability measure $\hat \rho_m, m\ge 1$ on ${\mathbb S}\times {\mathbb T}$
to probability measures $\tilde \rho_m$  on ${\mathbb R}^N\times {\mathcal W}\times {\mathbb R}^N$ with support on
${\mathbb S}\times {\mathbb T}$.
We observe that
\begin{eqnarray*}
\|f_{n}-f\|_{{\mathcal B}_1}
&\hskip-0.08in = & \hskip-0.08in \Big\|\sum_{m=n+1}^\infty
\|f_{m}-f_{m-1}\|_{{\mathcal B}_1} \int_{{\mathbb S}\times {\mathbb T}} {\bf a}^T
\sigma({\bf b}*{\bf x}+{\bf c}) \hat\rho_m(d{\bf a}, d{\bf b}, d{\bf c})\Big\|_{{\mathcal B}_1}\nonumber\\
& \hskip-0.08in \le & \hskip-0.08in \sum_{m=n+1}^\infty  \|f_{m}-f_{m-1}\|_{{\mathcal B}_1}
\Big\|\int_{{\mathbb S}\times {\mathbb T}} {\bf a}^T
\sigma({\bf b}*{\bf x}+{\bf c}) \hat\rho_m(d{\bf a}, d{\bf b}, d{\bf c})\Big\|_{{\mathcal B}_1}\nonumber\\
& \hskip-0.08in \le & \hskip-0.08in  \sum_{m=n+1}^\infty  \|f_{m}-f_{m-1}\|_{{\mathcal B}_1}
{\mathbb E}_{\tilde \rho_m} (\|{\bf a}\|_* (\|{\bf b}\|_{\rm co}+\|{\bf c}\|)) \le 2^{-n+1},
\end{eqnarray*}
where the equality holds by  \eqref{BanachBarronspace.thm.pfeq2},  \eqref{BanachBarronspace.thm.pfeq3} and \eqref{BanachBarronspace.thm.pfeq4},
and the first, second and third  inequality follows from Lemma \ref{quasinorm.lem},  the definition of Barron norm
and  \eqref{BanachBarronspace.thm.pfeq1} respectively.  Therefore $f_n, n\ge 1$, converges to $f\in {\mathcal B}_1$ and hence
${\mathcal B}_1$ is a Banach space.

 By \eqref{connorm0.def},  \eqref{connorm.def} and \eqref{quasinorm2.eq1},
 we have
 \begin{eqnarray*}
 |f({\bf x})| & \hskip-0.08in \le & \hskip-0.08in \|f\|_{{\mathcal B}_1}
 \int_{{\mathbb S}\times {\mathbb T}} |{\bf a}^T\sigma({\bf b}*{\bf x}+{\bf c})| \hat \rho(d{\bf a}, d{\bf b}, d{\bf c})\\
 & \hskip-0.08in \le & \hskip-0.08in \|f\|_{{\mathcal B}_1}
 \int_{{\mathbb S}\times {\mathbb T}} \|{\bf a}\|_* \|{\bf b}*{\bf x}+{\bf c}\| \hat \rho(d{\bf a}, d{\bf b}, d{\bf c})\\
 & \hskip-0.08in \le & \hskip-0.08in \|f\|_{{\mathcal B}_1}
 \int_{{\mathbb S}\times {\mathbb T}} \|{\bf a}\|_* (\|{\bf b}\|_{\rm co}+\|{\bf c}\|) \hat \rho(d{\bf a}, d{\bf b}, d{\bf c})
 = \|f\|_{{\mathcal B}_1}.
 \end{eqnarray*}
 This proves the reproducing kernel property \eqref{BanachBarronspace.eq1} for the Banach space ${\mathcal B}_1$.

 Applying H\"older inequality, we have
\begin{equation}\label{BanachBarronspace.thm.pfeq6}
\|f\|_{{\mathcal B}_1}\le
\|f\|_{{\mathcal B}_{r}}\le \|f\|_{{\mathcal B}_{\infty}} \ {\rm for \ all} \ f\in {\mathcal B}_\infty \ {\rm and} \ 1\le r\le \infty.
\end{equation}
Therefore the proof of the norm equivalence in \eqref{BanachBarronspace.eq2} reduces to establishing
\begin{equation}\label{BanachBarronspace.thm.pfeq7}
\|f\|_{{\mathcal B}_\infty}\le \|f\|_{{\mathcal B}_{1}} \ {\rm for \ all} \ f\in {\mathcal B}_1.
\end{equation}
By Lemma
 \ref{quasinorm2.lem}, there exist a probability  measure $\hat \rho$ on ${\mathbb S}\times {\mathbb T}$
 such that
\begin{eqnarray} \label{BanachBarronspace.thm.pfeq8}
f({\bf x})=\|f\|_{{\mathcal B}_1} \int_{{\mathbb S}\times {\mathbb T}} {\bf a}^T
\sigma({\bf b}*{\bf x}+{\bf c}) \hat\rho(d{\bf a}, d{\bf b}, d{\bf c}).
\end{eqnarray}
Dilate the measure $\hat\rho$ on ${\mathbb S}\times {\mathbb T}$
to a probability measure on $(\|f\|_{{\mathcal B}_1}{\mathbb S})\times {\mathbb T}$
and then extend  to a probability measure  $\tilde \rho$ on ${\mathbb R}^N\times {\mathcal W}\times {\mathbb R}^N$ with support on
$(\|f\|_{{\mathcal B}_1}{\mathbb S})\times {\mathbb T}$.
Then one may verify that
$\tilde \rho\in {\mathcal P}_f$ and
$$ \|f\|_{{\mathcal B}_\infty}\le \sup_{ ({\bf a}, {\bf b}, {\bf c})\in {\rm supp} \tilde \rho} \|{\bf a}\|_* (\|{\bf b}\|_{\rm co}+\|{\bf c}\|)=
\|f\|_{{\mathcal B}_1}.$$
This proves \eqref{BanachBarronspace.thm.pfeq7}.  Hence
the desired conclusion that ${\mathcal B}_r, 1\le r\le \infty$, are Banach spaces independent on $1\le r\le \infty$.
\end{proof}

\subsection{Proof of Theorem \ref{rkhs.thm}}
\label{rkhs.thm.pfsection}

Let ${\mathcal L}_{\hat\rho}^o$ be the linear space spanned by $ {\bf a}^T \sigma({\bf b}*{\bf x}+{\bf c}), {\bf x}\in \Omega$. One may verify that $g\in {\mathcal H}^o_{\hat \rho}$ if and only if $g=\sum_{i=1}^I c_i  K_{\hat \rho} (\cdot, {\bf x}_i)$ for some $c_i\in {\mathbb R}$ and ${\bf x}_i\in \Omega, 1\le i\le I$, if and only if
\begin{equation}
 \label{rkhs.thm.pfeq1}
g({\bf x})
=\int_{{\mathbb S}\times {\mathbb T}}
 {\bf a}^T \sigma({\bf b}*{\bf x}+{\bf c}) \eta({\bf a}, {\bf b}, {\bf c}) \hat\rho(d{\bf a}, d{\bf b}, d{\bf c})
\end{equation}
 for some function $\eta\in {\mathcal L}_{\hat\rho}^o$.
Moreover,
\begin{equation}  \label{rkhs.thm.pfeq2}
\|g\|_{{\mathcal H}_{\hat \rho}}= \Big( \int_{{\mathbb S}\times {\mathbb T}} |\eta({\bf a}, {\bf b}, {\bf c})|^2\hat\rho(d{\bf a}, d{\bf b}, d{\bf c})\Big)^{1/2}
\end{equation}
and
\begin{eqnarray} \label{rkhs.thm.pfeq3}
|g({\bf x})| &\hskip-0.08in \le &\hskip-0.08in
\int_{{\mathbb S}\times {\mathbb T}}
 |{\bf a}^T \sigma({\bf b}*{\bf x}+{\bf c})|
|\eta({\bf a}, {\bf b}, {\bf c})| \hat\rho(d{\bf a}, d{\bf b}, d{\bf c})\nonumber\\
& \hskip-0.08in \le & \hskip-0.08in
\int_{{\mathbb S}\times {\mathbb T}}
  |\eta({\bf a}, {\bf b}, {\bf c})| \hat\rho(d{\bf a}, d{\bf b}, d{\bf c})\nonumber\\
  & \hskip-0.08in   \le & \hskip-0.08in
\Big(\int_{{\mathbb S}\times {\mathbb T}}
  |\eta ({\bf a}, {\bf b}, {\bf c})|^2 \hat\rho(d{\bf a}, d{\bf b}, d{\bf c}) \Big)^{1/2}=\|g\|_{{\mathcal H}_{\hat \rho}}, \ {\bf x}\in \Omega.
\end{eqnarray}
This proves  \eqref{rkhs.thm.eq1}, \eqref{rkhs.thm.eq2} and
\eqref{rkhs.thm.eq3} for functions $g\in {\mathcal H}_{\hat \rho}^o$.
Recall that ${\mathcal H}_{\hat \rho}$  and ${\mathcal L}_{\hat \rho}^2$
are the completion of  ${\mathcal H}_{\hat \rho}^o$ and
 ${\mathcal L}_{\hat \rho}^o$  respectively.  Hence taking limits in
 \eqref{rkhs.thm.pfeq1}, \eqref{rkhs.thm.pfeq2}  and \eqref{rkhs.thm.pfeq3}, and
 using the reproducing kernel property of  ${\mathcal H}_{\hat \rho}$ and the conclusion in Remark \ref{rkhs.rem1} completes the proof.

\subsection{Proof of Theorem \ref{BarronRkhs.thm}}
\label{BarronRkhs.thm.pfsection}

Take $f\in {\mathcal B}$ and let $\hat \rho\in \hat P$ be the probability measure on ${\mathbb S}\times {\mathbb T}$ such that
\eqref{functionrepresentation.def2} holds.  Then by
Theorem \ref{rkhs.thm}, we conclude that $f\in  {\mathcal H}_{\hat \rho}$ and
$$\|f\|_{{\mathcal H}_{\hat \rho}}\le \|f\|_{\mathcal B}
\Big(\int_{{\mathbb S}\times {\mathbb T}} | (P_{\hat\rho} 1)({\bf a}, {\bf b}, {\bf c})|^2 \hat \rho(d{\bf a}, d{\bf b}, d{\bf c})\Big)^{1/2}\le
\|f\|_{\mathcal B}.$$
This shows that
\begin{equation}\label{BarronRkhs.thm.pfeq1}
{\mathcal B}\subset\cup_{{\hat \rho}\in \widehat P} {\mathcal H}_{\hat \rho}\
\ {\rm and} \ \
\inf _{f\in H_{\hat \rho}, \hat \rho\in \hat P} \|f\|_{{\mathcal H}_{\hat \rho}}\le \|f\|_{\mathcal B}.
\end{equation}

Let $f\in {\mathcal H}_{\hat\rho}$ for some $\hat\rho\in \hat P$ and $\eta\in {\mathcal L}_{\hat \rho}^2$ so that
\eqref{rkhs.thm.eq1} holds.  The existence of such a function $\eta$ follows from \eqref{rkhs.thm.projectioneq}
and Theorem \ref{rkhs.thm}. Moreover, we have
\begin{equation}  \label{BarronRkhs.thm.pfeq2}
\|f\|_{{\mathcal H}_{\hat \rho}}=\|\eta\|_{L^2_{\hat \rho}}.
\end{equation}

Define
a probability measure $\hat {\tilde \rho}$ on ${\mathbb S}\times {\mathbb T}$ by
\begin{equation}  \label{BarronRkhs.thm.pfeq3}
\hat {\tilde \rho}(A)= \frac{\int_{E_1\cap A}
 \eta({\bf a}, {\bf b}, {\bf c})  \hat \rho(d{\bf a}, d{\bf b}, d{\bf c})
 -\int_{E_2\cap \tilde A}
 \eta({\bf a}, {\bf b}, {\bf c})  \hat \rho(d{\bf a}, d{\bf b}, d{\bf c})}
 { \|\eta\|_{L^1_{\hat \rho}}},
\end{equation}
where $E_1=\{({\bf a}, {\bf b}, {\bf c})\in {\mathbb S}\times {\mathbb T}: \  \eta({\bf a}, {\bf b}, {\bf c})\ge 0\}$,
$E_2=\{({\bf a}, {\bf b}, {\bf c})\in {\mathbb S}\times {\mathbb T}: \  \eta({\bf a}, {\bf b}, {\bf c})<0\}$,
and $\tilde A=\{({\bf a}, {\bf b}, {\bf c})\in {\mathbb S}\times {\mathbb T}: \  (-{\bf a}, {\bf b}, {\bf c})\in A\}$.
By \eqref{rkhs.thm.eq1} and the definition  \eqref{BarronRkhs.thm.pfeq3} of the probability measure $\hat {\tilde \rho}\in \hat P$, we have
\begin{eqnarray*}
f({\bf x}) & \hskip-0.08in = & \hskip-0.08in  \int_{E_1}
 {\bf a}^T \sigma({\bf b}*{\bf x}+{\bf c})\eta({\bf a}, {\bf b}, {\bf c}) \hat\rho(d{\bf a}, d{\bf b}, d{\bf c})\nonumber\\
 & \hskip-0.08in & \hskip-0.08in
 + \int_{E_2}   (-{\bf a})^T \sigma({\bf b}*{\bf x}+{\bf c}) (-\eta)({\bf a}, {\bf b}, {\bf c}) \hat\rho(d{\bf a}, d{\bf b}, d{\bf c})\nonumber\\
  & \hskip-0.08in = & \hskip-0.08in
\|\eta\|_{L^1_{\hat \rho}}
   \int_{{\mathbb S}\times{\mathbb T}}
 {\bf a}^T \sigma({\bf b}*{\bf x}+{\bf c})\eta({\bf a}, {\bf b}, {\bf c}) \hat {\tilde \rho}(d{\bf a}, d{\bf b}, d{\bf c}).
\end{eqnarray*}
This implies that $f\in {\mathcal B}$ and
$\|f\|_{\mathcal B}\le \|\eta\|_{L^1_{\hat \rho}}$. This together with \eqref{BarronRkhs.thm.pfeq2}
and the observation $\|\eta\|_{L^1_{\hat \rho}}\le \|\eta\|_{L^2_{\hat \rho}}$ implies  that
\begin{equation} \label{BarronRkhs.thm.pfeq4}
{\mathcal H}_{\hat \rho}\subset {\mathcal B}\ \ {\rm and}\  \ \|f\|_{\mathcal B}\le \|f\|_{{\mathcal H}_{\hat \rho}} \ {\rm for \ all} \ f\in {\mathcal H}_{\hat \rho}.
\end{equation}

Combining \eqref{BarronRkhs.thm.pfeq2} and \eqref{BarronRkhs.thm.pfeq4} completes the proof.

\subsection{Proof of Theorem \ref{approximation.thm}}
\label{approximation.thm.pfsection}

 Without loss of generality, we assume that $f\in {\mathcal B}$ is a nonzero function with $\|f\|_{\mathcal B}=1$, otherwise replacing
$f$ by $f/\|f\|_{\mathcal B}$.
Then by Lemma \ref{quasinorm2.lem}, there exists a probability measure $\hat \rho$ on ${\mathbb S}\times {\mathbb T}$ such that
\begin{equation} \label{approximation.thm.pfeq1}
f({\bf x})=  \int_{{\mathbb S}\times {\mathbb T}} {\bf a}^T \sigma({\bf b}*{\bf x}+{\bf c}) \hat\rho(d{\bf a}, d{\bf b}, d{\bf c}), \ {\bf x}\in \Omega.
\end{equation}
Let ${\pmb \theta}_m=({\bf a}_m, {\bf b}_m, {\bf c}_m)\in {\mathbb S}\times {\mathbb T}, 1\le m\le M$,
be i.i.d. random variables following the probability measure $\hat \rho$.
Set  ${\pmb \Theta}=({\pmb \theta}_1, \ldots, {\pmb \theta}_M)$ and define
$$f_M({\bf x}, {\pmb \Theta})=\frac{1}{M} \sum_{m=1}^M \phi ({\bf x}, \theta_m), \ \ {\bf x}\in \Omega. $$   

Clearly, we have
$$\|{\pmb \Theta}\|_{P, \infty}= 1=\|f\|_{\mathcal B}.$$
Observe that  $ \phi({\bf x}, \pmb \theta_m)= {\bf a}_m^T \sigma  ({\bf b}_m*{\bf x}+{\bf c}_m)$
 are i.i.d random variables with
$\phi({\bf x}, \pmb\theta_m)\in [-1, 1]$ almost surely and
${\mathbb E} \phi({\bf x}, \pmb \theta_m)= f({\bf x})$ hold for all $1\le m\le M$ and ${\bf x}\in \Omega$.
Combining the above observation with   \eqref{fM.def} and \eqref{approximation.thm.pfeq1}, we obtain
\begin{eqnarray*}
{\mathbb E} \int_{\Omega}  ( f_M({\bf x}, {\pmb\Theta})- f({\bf x}))^2 d\mu({\bf x})
     &\hskip-0.08in = & \hskip-0.08in
     \frac{1}{M} \int_{\Omega}    {\mathbb E} (  \phi({\bf x}, \pmb \theta) -  {\mathbb E} \phi({\bf x}, \pmb \theta))^2  d\mu({\bf x})\nonumber\\
     &\hskip-0.08in \le  & \hskip-0.08in
     \frac{1}{M} \int_{\Omega}
      {\mathbb E} (  \phi({\bf x}, \pmb \theta))^2  d\mu({\bf x}) \le  \frac{1}{M}.
\end{eqnarray*}
Applying the Markov's inequality yields
\begin{equation*}
{\mathbb P}\Big\{  \int_{\Omega}  ( f_M({\bf x}, {\pmb\Theta})- f({\bf x}))^2 d\mu({\bf x})>  \frac{1+\epsilon}{M}\Big\}\le \frac{1}{1+\epsilon}<1.
\end{equation*}
This completes the proof.

\subsection{Proof of Theorem \ref{approximation.thm2}}
\label{approximation.thm2.pfsection}

We follow the argument of Theorem \ref{approximation.thm}.  Without loss of generality, we assume that $\|f\|_{\mathcal B}=1$.
Let $\hat \rho$ be the probability measure on ${\mathbb S}\times {\mathbb T}$ in \eqref{approximation.thm.pfeq1},
${\pmb \Theta}=({\pmb \theta}_1, \ldots, {\pmb \theta}_M)$ be the i.i.d random variables following the probability measure $\hat \rho$, and define
$f_M({\bf x}, {\pmb \Theta})$  be
as in \eqref{fM.def}.

Set $I=N_\epsilon^{\rm ext}$ and take a family of  balls $B({\bf x}_i, \epsilon), 1\le i\le I$,   with center ${\bf x}_i\in \Omega$ and radius $\epsilon$ that covers the domain $\Omega$,
\begin{equation} \label{approximation.thm2.pfeq0}
\cup_{i=1}^I B({\bf x}_i, \epsilon)=\Omega.
\end{equation}
For any random variable ${\pmb \theta}=({\bf a}, {\bf b}, {\bf c})\in {\mathbb S}\times {\mathbb T}$,
we obtain from \eqref{fM.def} and the definition of ${\mathbb S}\times {\mathbb T}$ that
$\phi({\bf x}_i, {\pmb \theta})= {\bf a}^T \sigma({\bf b}*{\bf x}_i+{\bf c})\in [-1, 1]$ almost surely and  
 ${\mathbb E} \phi({\bf x}_i, {\pmb \theta})=f({\bf x}_i)$.
Therefore $\phi({\bf x}_i, \theta)-f({\bf x}_i)$ is subGaussian with variance proxy 1
by Hoeffding’s inequality, i.e.,
\begin{equation*}
{\mathbb E} \exp \big(  s( \phi({\bf x}_i, {\pmb \theta})-f({\bf x}_i))\big)\le \exp (s^2/2), \ s\in {\mathbb R}.
\end{equation*}
Therefore for all $t>0$, we have
\begin{equation}\label{approximation.thm2.pfeq1}
{\mathbb P}\{ |f_M({\bf x}_i, {\pmb \Theta})-f({\bf x}_i)|> t \}\le 2 \exp(- M t^2/2).
\end{equation}
This  together with  \eqref{approximation.thm2.eq0}
implies the existence of ${\pmb \Theta}^*\in ({\mathbb S}\times {\mathbb T})^M$ such that
\begin{equation} \label{approximation.thm2.pfeq2}
\|{\pmb \Theta}^*\|_{P, \infty}\le  1=\|f\|_{\mathcal B}\end{equation}
and
\begin{equation}
\label{approximation.thm2.pfeq3}
 |f_M({\bf x}_i, {\pmb \Theta}^*)-f({\bf x}_i)|\le  \epsilon, \ 1\le i\le I.
 \end{equation}

For the above  shallow GCNN with parameter ${\pmb \Theta}^*$,
we obtain from  \eqref{fM.eq1},  \eqref{approximation.thm2.pfeq0},
 \eqref{approximation.thm2.pfeq3}, and Corollary
 \ref{lipschitz.cor}
that
\begin{eqnarray*}
& &
|f_M({\bf x}, {\pmb \Theta}^*)-f({\bf x})| \nonumber\\
& \hskip-0.08in \le & \hskip-0.08in
 |f_M({\bf x}, {\pmb \Theta}^*)-f_M({\bf x}_i, {\pmb \Theta}^*)|+|f_M({\bf x}_i, {\pmb \Theta}^*)-f({\bf x}_i)|+|f({\bf x})-f({\bf x}_i)|\nonumber\\
& \hskip-0.08in \le & \hskip-0.08in
D_1 \|\sigma\|_{\rm Lip} \epsilon \|{\pmb \Theta}^*\|_{P, \infty}+ \epsilon +
D_1 \|\sigma\|_{\rm Lip} \epsilon \|f\|_{\mathcal B} \le  (2D_1 \|\sigma\|_{\rm Lip}+1) \epsilon, \ {\bf x}\in \Omega,
\end{eqnarray*}
where  ${\bf x}_i$ is chosen so that ${\bf x}\in B({\bf x}_i, \epsilon)$.
This together with \eqref{approximation.thm2.pfeq2} completes the proof.

\subsection{Proof of Theorem \ref{inverseapproximaton.thm}}
\label{inverseapproximation.thm.pfsection}

By the assumption $f_n\in {\mathcal C}_{Q, p}, n\ge 1$, we can write
 \begin{equation} \label{inverseapproximaton.thm.pfeq0}
 f_n(x)=\frac{1}{M_n} \sum_{m=1}^{M_n}  \alpha_{n, m}
  {\bf a}_{n, m}^T \sigma ( {\bf  b}_{n,m} * {\bf x} +{\bf c}_{n,m} ), \ {\bf x}\in \Omega,
 \end{equation}
 where $({\bf a}_{n, m}, {\bf b}_{n, m}, {\bf c}_{n, m})\in {\mathbb S}\times  {\mathbb T}, 1\le m\le M_n$, and
 $0<M_n^{-1}\sum_{m=1}^{M_n}\alpha_{n,m}\le Q$.
Define the probability measure $\hat \rho_n$ on ${\mathbb S}\times {\mathbb T}$ by
    \begin{equation*}
        \rho_n (d{\bf a}, d{\bf b}, d{\bf c})= \Big(\sum_{m=1}^{M_n} \alpha_{n, m}\Big)^{-1}
         \sum_{m=1}^{M_n}\alpha_{n, m} \delta \big({\bf a}-{\bf a}_{m,n}, {\bf b}-{\bf b}_{m,n}, {\bf c}- {\bf c}_{m,n}\big),
    \end{equation*}
where $\delta$ is the Dirac measure centered at  the origin.
Then we rewrite the representation in \eqref{inverseapproximaton.thm.pfeq0} as follows:
    \begin{equation} \label{inverseapproximaton.thm.pfeq1}
        f_n({\bf x}) = \frac{\sum_{m=1}^{M_n} \alpha_{n, m}}{M_n}
        \int_{{\mathbb S}\times {\mathbb T}} {\bf a}^T\sigma ( {\bf b}* {\bf x} +{\bf  c}) \hat \rho_n(d{\bf a},d{\bf b},d{\bf c}), \ {\bf x}\in \Omega.
    \end{equation}

As  $\sum_{m=1}^{M_n} \alpha_{n,m}/M_n, n\ge 1$, is  a bounded sequence   contained in $[0, Q]$, without loss of generality, we assume that
it is convergent,
\begin{equation}\label{inverseapproximaton.thm.pfeq2}
\lim_{n\to \infty} \frac{1}{M_n} \sum_{m=1}^{M_n} \alpha_{n, m}= \alpha\in [0, Q],
\end{equation}
 otherwise replacing  $f_n, n\ge 1$, by its appropriate subsequence.

Observe that the sequence $\hat \rho_n, n\ge 1$, of probability measures on   ${\mathbb S}\times {\mathbb T}$. By applying Prokhorov's theorem \cite{Prokhorov1956},
there exists a subsequence of probability measures  $\hat \rho_n$ on ${\mathbb S}\times {\mathbb T}$ that converges weakly to some
 measure $ {\hat \rho}$ on ${\mathbb S}\times {\mathbb T}$.  Without loss of generality, we assume that
 the original sequence  $\hat \rho_n, n\ge 1$, of probability measures  converges weakly to some measure $\hat \rho$ on ${\mathbb S}\times {\mathbb T}$, i.e.,
 \begin{equation}  \label{inverseapproximaton.thm.pfeq3}
 \lim_{n\to \infty} \int_{{\mathbb S}\times {\mathbb T}} g({\bf a}, {\bf b}, {\bf c}) \hat \rho_{n} (d{\bf a}, d{\bf b}, d{\bf c})=
 \int_{{\mathbb S}\times {\mathbb T}} g({\bf a}, {\bf b}, {\bf c}) \hat \rho (d{\bf a}, d{\bf b}, d{\bf c})
  \end{equation}
hold for all continuous functions $g$ on ${\mathbb S}\times {\mathbb T}$.

As $\hat \rho_n, n\ge 1$, are probability measures on ${\mathbb S}\times {\mathbb T}$ and the constant function is a continuous function on ${\mathbb S}\times {\mathbb T}$, we
obtain from \eqref{inverseapproximaton.thm.pfeq3}
that
$\hat \rho$ is also a  probability measure on ${\mathbb S}\times {\mathbb T}$.
Define
\begin{equation} \label{inverseapproximaton.thm.pfeq4}
\tilde f({\bf x})= \alpha \int_{{\mathbb S}\times {\mathbb T}} {\bf a}^T \sigma({\bf b}*{\bf x}+{\bf c}) \hat \rho( d{\bf a}, d{\bf b}, d{\bf c}), \ {\bf x}\in \Omega.
\end{equation}
Then
the function $\tilde f$ defined by \eqref{inverseapproximaton.thm.pfeq4} belongs to the Barron space ${\mathcal B}$ and  has its norm
bounded by $\alpha\le Q$.

Recall that ${\bf a}^T \sigma({\bf b}*{\bf x}+{\bf c})$ are continuous functions about $({\bf a}, {\bf b}, {\bf c})\in {\mathbb S}\times {\mathbb T}$ for all ${\bf x}\in \Omega$. This together with \eqref{inverseapproximaton.thm.pfeq1}, \eqref{inverseapproximaton.thm.pfeq2}, \eqref{inverseapproximaton.thm.pfeq3} and
\eqref{inverseapproximaton.thm.pfeq4} implies that
\begin{equation}
\tilde f({\bf x})=\lim_{n\to \infty} f_n({\bf x}), \ {\bf x}\in \Omega.
\end{equation}
This proves that $\tilde f=f$ and hence completes the proof.

\subsection{Proof of Theorem \ref{density.thm}}
\label{density.thm.pfsection}

%
By  the universal approximation theorem in Lemma \ref{universal.lem}, it suffices to prove that for any ${\bf v}\in {\mathbb R}^N$ and $w\in {\mathbb R}$  there exist
${\bf a}, {\bf b}, {\bf c}\in {\mathbb R}^N$ such that
\begin{equation}
\label{density.thm.pfeq1}
\sigma({\bf v}^T {\bf x}+w)= {\bf a}^T \sigma({\bf b}*{\bf x}+{\bf c}),  \ {\bf x}\in \Omega.
\end{equation}
Take ${\bf a}={\bf e}_{n_0}$ and ${\bf c}=w {\bf e}_{n_0}$, where ${\bf e}_{n_0}$ is the unit vector in ${\mathbb R}^N$ with zero components except one for its $n_0$-th component.
Then the existence of the vector ${\bf b}$ in \eqref{density.thm.pfeq1} reducing to showing
\begin{equation} \label{density.thm.pfeq2}
{\bf e}_{n_0}^T ({\bf b}*{\bf x})= {\bf v}^T {\bf x} \ \ {\rm for \ all} \ {\bf x}\in {\mathbb R}^N.
\end{equation}
By \eqref{convolution.def3}, we need to find a multivariate polynomial $h$
such that
$$ h({\bf S}_1, \ldots, {\bf S}_K) {\bf e}_{n_0}= {\bf v},
$$
or equivalently,
\begin{equation}
{\rm diag} \big(h({\pmb \lambda}(1)), \ldots, h({\pmb \lambda}(N))\big)
 {\bf U}^T {\bf e}_{n_0}= {\bf U}^T {\bf v}.\end{equation}
The above equation about the polynomial $h$ is solvable as all entries of ${\bf U}^T {\bf e}_{n_0}$ are nonzero
and ${\pmb \lambda}(n), 1\le n\le N$ in the joint spectrum of graph shifts are distinct by Assumption \ref{shiftassumption}. In particular,  $h$ is
an interpolation polynomial  satisfying
$$ h({\pmb \lambda}(n))= \frac{ \big({\bf U}^T {\bf v}\big)(n)}{ \big({\bf U}^T {\bf e}_{n_0}\big)(n)}, \ 1\le n\le N$$
where ${\bf y}(n)$ is the $n$-th component of a vector ${\bf y}\in {\mathbb R}^N$.
This completes the proof.

\subsection{Proof of  Theorem \ref{Rademacher.thm}}
\label{Rademacher.thm.pfsection}

To prove  Theorem \ref{Rademacher.thm}, we recall the contraction lemma and  Massart lemma, where
$\xi_i, 1\le i\le S$, are i.i.d. Rademacher random variables \cite{Shai2014}.

\begin{lemma}
\label{contraction.lem}
 Let ${\mathcal K}$ be a family of functions on $\Omega$ and ${\bf x}_i\in \Omega, 1\le i\le S$.
Then
\begin{equation}
{\mathbb E} \sup_{g\in {\mathcal K}} \sum_{i=1}^S \xi_i \sigma\big(g({\bf x}_i)\big)\le
{\mathbb E} \sup_{g\in {\mathcal K}} \sum_{i=1}^S \xi_i g({\bf x}_i).
\end{equation}
\end{lemma}


\begin{lemma} \label{massart.lem}
 Let ${\mathcal T}\subset {\mathbb R}^S$ be a finite set with its cardinality denoted by $\# {\mathcal T}$. Then
 $${\mathbb E} \Big(\max_{{\bf t}\in {\mathcal T}}\sum_{i=1}^S  \xi_i t_i\Big)
 \le  \sqrt{2\ln \# {\mathcal T}}  \max_{{\bf t}\in {\mathcal T}}\|{\bf t}\|_2,
$$
where ${\bf t}=[t_1, \ldots, t_S]\in {\mathcal T}$.
\end{lemma}

Now we are ready to prove Theorem \ref{Rademacher.thm}.

\begin{proof} [Proof of Theorem \ref{Rademacher.thm}]
First we show that
\begin{equation} \label {Rademacher.thm.pfeq1}
\sup_{f\in {\mathcal F}_Q} \sum_{i=1}^S \xi_i  f({\bf x}_i)\le  Q \sup_{({\bf b}, {\bf c})\in  {\mathbb T}} \Big\| \sum_{i=1}^S \xi_i \sigma({\bf b}*{\bf x}_i+{\bf c})\Big\|_\infty
\end{equation}
hold for all $\xi_i\in \{-1, 1\}$ and ${\bf x}_i\in \Omega, 1\le i\le S$.

By Lemma \ref{quasinorm2.lem}, there exists  a probability measure $\hat {\rho}$ on ${\mathbb S}\times {\mathbb T}$  for any $f\in {\mathcal B}$ such that
$$f({\bf x}) = \|f\|_{\mathcal B} \int_{{\mathbb S}\times {\mathbb  T}} {\bf a}^T \sigma ({\bf  b} *{\bf  x} + {\bf c}) \hat{\rho} (d{\bf a}, d{\bf b}, d{\bf c}).
$$
Then
    \begin{eqnarray*}
\sum_{i=1}^S \epsilon_i  f({\bf x}_i)
    &\hskip-0.08in = & \hskip-0.08in  \|f\|_{\mathcal B}
     \int_{{\mathbb S}\times {\mathbb T}}
     {\bf a}^T \Big(\sum_{i=1}^S \xi_i\sigma ({\bf b} * {\bf x}_i +{\bf  c})  \Big) \hat{\rho} (d{\bf a}, d{\bf b}, d{\bf c}) \\
        & \hskip-0.08in\le & \hskip-0.08in \|f\|_{\mathcal B}
        \sup_{{\bf a}\in {\mathbb S}, ({\bf b}, {\bf c})\in {\mathbb T}}
         {\bf a}^T \Big(\sum_{i=1}^S \xi_i\sigma ({\bf b} * {\bf x}_i +{\bf  c}) \Big)
 \notag\\
        &\hskip-0.08in \le &  \hskip-0.08in \|f\|_{\mathcal B}
                \sup_{ ({\bf b}, {\bf c})\in {\mathbb T}}
         \Big\| \sum_{i=1}^S \xi_i\sigma ({\bf b} * {\bf x}_i +{\bf  c}) \Big\|_\infty.
    \end{eqnarray*}
    Taking supremum over all $f\in {\mathcal F}_Q$ in the above estimate proves \eqref{Rademacher.thm.pfeq1}.

 \medskip

Next we use \eqref{Rademacher.thm.pfeq1} and apply the contraction lemma  to show  that
\begin{equation} \label {Rademacher.thm.pfeq2}
{\mathbb E} \sup_{ ({\bf b}, {\bf c})\in {\mathbb T}}
         \Big\| \sum_{i=1}^S \xi_i\sigma ({\bf b} * {\bf x}_i +{\bf  c}) \Big\|_\infty
\le 2D_2 {\mathbb  E} \Big\|\sum_{i=1}^S \xi_i {\bf x}_i\Big\|_\infty+ 2 {\mathbb E} \Big|\sum_{i=1}^S \xi_i\Big|.
\end{equation}

For a vector ${\bf y}\in {\mathbb R}^N$, we denote its $n$-th component by ${\bf y}(n), 1\le n\le N$. Observe that
\begin{eqnarray*}   
      \sup_{ ({\bf b}, {\bf c})\in {\mathbb T}}
         \Big\| \sum_{i=1}^S \xi_i\sigma ({\bf b} * {\bf x}_i +{\bf  c}) \Big\|_\infty
    & \hskip-0.08in =  &  \hskip-0.08in \sup_{ ({\bf b}, {\bf c})\in {\mathbb T}}
     \max_{1\le n\le N}\max_{\xi\in \{-1, 1\}}
      \Big(\xi\sum_{i=1}^S \xi_i\sigma ({\bf b} * {\bf x}_i +{\bf  c}) \Big)(n) \nonumber
\\
    &\hskip-0.08in \le  & \hskip-0.08in
        \sum_{\xi\in \{-1, 1\}}  \max_{1\le n\le N} \sup_{ ({\bf b}, {\bf c})\in {\mathbb T}}
      \Big(\xi \sum_{i=1}^S \xi_i\sigma ({\bf b} * {\bf x}_i +{\bf  c}) \Big)(n)
\end{eqnarray*}
where the inequality holds as
$$\sup_{ ({\bf b}, {\bf c})\in {\mathbb T}}
      \Big(\pm \sum_{i=1}^S \xi_i\sigma ({\bf b} * {\bf x}_i +{\bf  c}) \Big)(n)\ge
  \sup_{ \|{\bf c}\|_\infty=1}
     \pm \sum_{i=1}^S \xi_i \sigma ({\bf  c}(n)) \ge
       0,\   1\le n\le N.$$
     This together with  Lemma \ref{contraction.lem} with ${\mathcal K}=\{({\bf b} * {\bf x} +{\bf  c})(n), 1\le n\le N, ({\bf b}, {\bf c})\in {\mathbb T}\}$  implies that
      \begin{eqnarray*}
  {\mathbb E} \sup_{ ({\bf b}, {\bf c})\in {\mathbb T}}
         \Big\| \sum_{i=1}^S \xi_i\sigma ({\bf b} * {\bf x}_i +{\bf  c}) \Big\|_\infty 
&\hskip-0.08in   \le  & \hskip-0.08in 2
{\mathbb E}\max_{1\le n\le N} \sup_{ ({\bf b}, {\bf c})\in {\mathbb T}}
 \sum_{i=1}^S \xi_i \sigma\big( ({\bf b} * {\bf x}_i +{\bf  c})(n)\big)\nonumber\\
&\hskip-0.08in   \le  & \hskip-0.08in 2
              \E \max_{1\le n\le N} \sup_{ ({\bf b}, {\bf c})\in {\mathbb T}}
 \sum_{i=1}^S \xi_i ({\bf b} * {\bf x}_i +{\bf  c}) (n)\nonumber\\
&\hskip-0.08in   \le  & \hskip-0.08in 2
              \E  \sup_{ ({\bf b}, {\bf c})\in {\mathbb T}}
 \Big\|{\bf b}* \Big(\sum_{i=1}^S \xi_i {\bf x}_i\Big) + \Big(\sum_{i=1}^S \xi_i \Big) {\bf  c} \Big\|_\infty.
\end{eqnarray*}
 Combining the above estimate  with \eqref{connorm.def2} and the definition of the set ${\mathbb T}$, we complete the proof of
\eqref{Rademacher.thm.pfeq2}.

\medskip

Finally we apply \eqref{Rademacher.thm.pfeq2} and use the Massart Lemma to prove \eqref{Rademacher.thm.eq1}.

Observe that
$$\Big\|\sum_{i=1}^S \xi_i {\bf x}_i\Big\|_\infty= \max_{1\le n\le N} \max\Big \{\sum_{i=1}^S \xi_i {\bf x}_i(n), \sum_{i=1}^S \xi_i (-{\bf x}_i)(n)\Big\},$$
where ${\bf x}_i(n), 1\le n\le N$, are the $n$-th entries of ${\bf x}_i\in {\mathbb R}^N$.
Applying Lemma \ref{massart.lem} with ${\mathcal T}=\{ [{\bf x}_1(n), \ldots, {\bf x}_S(n)], 1\le n\le N\} \cup
\{-[{\bf x}_1(n), \ldots, {\bf x}_S(n)], 1\le n\le N\}$, we conclude that
$$
 {\mathbb E}  \Big\|\sum_{i=1}^S \xi_i {\bf x}_i\Big\|_\infty  \le \sqrt{2\ln (2N)} \sup_{1\le n\le N} \Big(\sum_{i=1}^S({\bf x}_i(n))^2\Big)^{1/2}.$$
 This together with \eqref{domain.def-1} implies that
\begin{equation}  \label{Rademacher.thm.pfeq4}
 {\mathbb E}  \Big\|\sum_{i=1}^S \xi_i {\bf x}_i\Big\|_\infty  \le D_0\sqrt{2\ln (2N)} S^{1/2} .
 \end{equation}

Let ${\bf 1}_S$ be the $S$-dimensional vector with all components taking value one.
Applying Lemma \ref{massart.lem} with ${\mathcal T}=\{-{\bf 1}_S, {\bf 1}_S\}$ gives
\begin{equation}\label {Rademacher.thm.pfeq5}
 {\mathbb E}  \Big|\sum_{i=1}^S \xi_i \Big|\le  \sqrt{ 2\ln 2} S^{1/2}. \end{equation}

Combining \eqref{Rademacher.thm.pfeq2}, \eqref{Rademacher.thm.pfeq4} and \eqref{Rademacher.thm.pfeq5} proves the desired estimate  \eqref{Rademacher.thm.eq1} on
the Rademacher complexity.
\end{proof}

\subsection{Proof of Theorem \ref{error.thm}}
\label{error.thm.pfsection}

 Set ${\bf X}=({\bf x}_1, \ldots, {\bf x}_S)$  and let $\Phi({\bf X})$ be as in \eqref{error.thm.eq1}.  By the symmetry of the set ${\mathcal F}_Q$, we have
$$\Phi({\bf X})=\sup_{f\in {\mathcal F}_Q} \int_{\Omega} f({\bf x}) d\mu({\bf x})- \frac{1}{S} \sum_{i=1}^S f({\bf x}_i) .$$
By the reproducing kernel property \eqref{BanachBarronspace.eq1} for the Barron space, we have
$$|\Phi({\bf X}))- \Phi({\bf X}_i^\prime)|\le S^{-1}
|f({\bf x}_i)-f({\bf x}_i^\prime)|\le 2QS^{-1}, \ 1\le i\le S, $$
where for $1\le i\le S$,   ${\bf X}$ and  ${\bf X}_i^\prime$  share the same components except that their $i$-th components are ${\bf x}_i$ and ${\bf x}_i^\prime$ respectively.
Then applying McDiarmid's inequality, 
we have that
$$\Phi( {\bf  X})\le  {\mathbb E}_{\bf X} \Phi( {\bf X})+
 \sqrt {2\ln (1/\delta)} Q S^{-1/2}$$
 with probability at least $1-\delta$.
Then by Theorem \ref{Rademacher.thm}, it suffices to prove
\begin{equation}\label{error.thm.pfeq1}
{\mathbb E}_{\bf X} \Phi( {\bf X})\le 2 {\rm Rad}_S({\mathcal F}_Q).
\end{equation}

Let us  draw a second sample ${\bf x}_1^\prime , \ldots, {\bf x}_S^\prime$ according to  probability measure $\mu$,
$\xi_i\in \{-1, 1\}, 1\le i\le S$ be i.i.d. Rademacher random variables
        with $P(\xi_i= 1)=P(\xi_i=-1)=1/2$, and set $\Xi=(\xi_1, \ldots, \xi_S)$.
Then
\begin{eqnarray*}
{\mathbb E}_{\bf X}
\Phi({\bf X}) &\hskip-0.08in  = & \hskip-0.08in
{\mathbb E}_{\bf  X} \sup_{f\in {\mathcal F}_Q}  {\mathbb E}_{{\bf X}'}  \frac{1}{S} \sum_{i=1}^S  \big(f({\bf x}_i^{\prime})-  f({\bf x}_i)\big)\nonumber\\
 &\hskip-0.08in  \le & \hskip-0.08in {\mathbb E}_{\bf X}  {\mathbb E}_{{\bf X}'} \sup_{f\in {\mathcal F}_Q}   \frac{1}{S} \sum_{i=1}^S  \big(f({\bf x}_i^{\prime})-  f({\bf x}_i)\big)
\nonumber\\
 &\hskip-0.08in  = & \hskip-0.08in  E_{\Xi}
 {\mathbb E}_{\bf X}  {\mathbb E}_{{\bf X}'} \sup_{f\in {\mathcal F}_Q}   \frac{1}{S} \sum_{i=1}^S \xi_i \big(f({\bf x}_i^{\prime})-  f({\bf x}_i)\big)\nonumber
\\
 &\hskip-0.08in  \le  & \hskip-0.08in   E_{\Xi}  {\mathbb E}_{\bf X}   \sup_{f\in {\mathcal F}_Q} \frac{1}{S} \sum_{i=1}^S \xi_i f({\bf x}_i)+
  E_{\Xi}   {\mathbb E}_{{\bf X}'} \sup_{f\in {\mathcal F}_Q} \frac{1}{S} \sum_{i=1}^S (-\xi_i) f({\bf x}_i^\prime)= 2 {\rm Rad}_S({\mathcal F}_Q).
\end{eqnarray*}
This proves \eqref{error.thm.pfeq1} and then completes the proof of Theorem \ref{error.thm}.

\section{Conclusion and discussions}
\label{conclusion.section}

In this paper, we introduce a Barron space ${\mathcal B}$ associated with two-layer GCNNs in the spectral convolution setting and show that functions in the Barron space can be well approximated by outputs of GCNNs without suffering from the curse of dimensionality (the order of the underlying graph).

For a graph filter ${\bf G}= (G(i,j))_{i,j\in V}$,
define its geodesic-width  $\omega({\bf G})$
 by the smallest nonnegative integer
 such that  $G(i,j)=0$ for all $i,j\in V$ with $\rho(i,j)>\omega({\bf G})$.
   Denote the set of all matrices with their geodesic-width no larger than $L$ by ${\mathcal M}_L$.
   In the spatial approach to define  graph convolution,
   a localized matrix operation ${\bf x}\longmapsto {\bf H}{\bf x}$ associated with some matrix  ${\bf H}\in {\mathcal M}_L$
   is applied instead of the spectral convolution ${\bf x}\longmapsto {\bf b}*{\bf x}$ associated with a graph signal ${\bf b}$.
    In particular, given a graph signal ${\bf b}$ in the convolution space ${\mathcal W}$,
     we can find a polynomial filter ${\bf H}=h({\mathbf S}_1, \ldots, {\mathbf S}_K)$ in some ${\mathcal M}_L$ such that the corresponding spectral  convolution can be implemented by polynomial filtering procedure, i.e., ${\bf b}*{\bf x}={\bf H}{\bf x}$ holds for any graph signal ${\bf x}$, where
    $L$ is the degree of the multivariate polynomial $h$.
     Comparing with the spectral convolution setting, the  convolution  in the spatial  setting has much more parameters to learn, as the convolution space ${\mathcal W}$ has $\dim {\mathcal W}\le N$, while the convolution space ${\mathcal M}_L$ in the spatial setting has
    dimension bounded below by $N$ and above by $N^2$, i.e., $N\le \dim {\mathcal M}_L\le \dim {\mathcal M}_D=N^2$, where $D=\max_{i,j\in V} \rho(i,j)$ is the diameter of the graph ${\mathcal G}$.

   In the  spatial convolution setting, the output of a two-layer GCNN is given by
   $$f_M(x, {\pmb \Theta})=\frac{1}{M} \sum_{m=1}^M {\bf a}_m^T \sigma ({\bf H}_m{\bf x}+{\bf c})$$
   where ${\pmb \Theta}=({\pmb \theta}_1, \ldots, {\pmb \theta}_M)$  and ${\pmb \theta}_m=({\bf a}_m, {\bf H}_m, {\bf c}_m)\in {\mathbb R}^N\times {\mathcal M}_L\times {\mathbb R}^N, 1\le m\le M$, see \cite{Bruna2014, Kipf2017} for $L=1$.
   With appropriate  convolution norm for matrices in ${\mathcal M}_L$, we can define a Barron space associated with
   two-layer GCNNs in the spatial convolution setting for functions with the following representation
   $$f({\bf x})=\int_{{\mathbb R}^N \times {\mathcal M}_L\times {\mathbb R}^N}
   {\bf a}^T \sigma ({\bf H}{\bf x}+{\bf c}) \rho(d{\bf a}, d{\bf H}, d{\bf c}), \ {\bf x}\in \Omega,$$
and show that functions in the Barron space can be approximated by two-layer GCNNs,    where $\rho$ is a probability measure on ${\mathbb R}^N \times {\mathcal M}_L\times {\mathbb R}^N$, cf. \eqref{functionrepresentation.def}, \eqref{barronnorm.def} and \eqref{Barron.def00}.


\end{document}